\documentclass[twoside]{article}
\usepackage[utf8]{inputenc}
\usepackage{hyperref}
\usepackage{algorithm}
\usepackage{algorithmic}
\usepackage{amsmath,amsthm,amssymb}
\usepackage{graphicx}
\newtheorem{lemma}{Lemma}

\newtheorem{theorem}{Theorem}

\newtheorem{assumption}{Assumption}
\newtheorem{definition}{Definition}
\newtheorem{remark}{Remark}
\usepackage{tikz}
\usetikzlibrary{arrows}
\tikzstyle{var}=[circle,draw=black,fill=white,thin,minimum size=18pt,inner sep=0pt]
\tikzstyle{ivar}=[circle,draw=red,fill=white,thin,minimum size=18pt,inner sep=0pt]
\tikzstyle{avar}=[circle,draw=blue,fill=white,thin,minimum size=18pt,inner sep=0pt]
\tikzstyle{varh}=[circle,draw=gray,fill=white,thin,minimum size=18pt,inner sep=0pt,dashed]
\tikzstyle{arr}=[->,>=stealth',draw=black,thick]
\tikzstyle{arrh}=[->,>=stealth',draw=gray,thick,dashed]
\tikzstyle{biarr}=[<->,>=stealth',draw=black,fill=black,thick]
\tikzstyle{biarrh}=[<->,>=stealth',draw=gray,fill=gray,thick]
\usepackage[accepted]{aistats2021}
\begin{document}

%

%

\twocolumn[

\aistatstitle{High-Dimensional Feature Selection for Sample Efficient Treatment Effect Estimation}

\aistatsauthor{ Kristjan Greenewald \And Dmitriy Katz-Rogozhnikov \And  Karthik Shanmugam }

\aistatsaddress{ MIT-IBM Watson AI Lab \And  IBM Research \And IBM Research } ]

\begin{abstract}
The estimation of causal treatment effects from observational data is a fundamental problem in causal inference. 
To avoid bias, the effect estimator must control for all confounders. Hence practitioners often collect data for as many covariates as possible to raise the chances of including the relevant confounders. While this addresses the bias, this has the side effect of significantly increasing the number of data samples required to accurately estimate the effect due to the increased dimensionality. In this work, we consider the setting where out of a large number of covariates $X$ that satisfy strong ignorability, an unknown sparse subset $S$ is sufficient to include to achieve zero bias, i.e. $c$-equivalent to $X$. We propose a common objective function involving outcomes across treatment cohorts with nonconvex joint sparsity regularization that is guaranteed to recover $S$ with high probability under a linear outcome model for $Y$ and subgaussian covariates for each of the treatment cohort. This improves the effect estimation sample complexity so that it scales with the cardinality of the sparse subset $S$ and $\log |X|$, as opposed to the cardinality of the full set $X$. We validate our approach with experiments on treatment effect estimation.
\end{abstract}

\section{Introduction}
Consider the problem of estimating the treatment effect of $T$ on a univariate outcome $Y$ in the presence of (possibly confounding) covariates $X$, where the treatment variable can take $q$ possible treatment configurations. We assume only observational data is available. The causal graph for this setup is shown in Figure \ref{fig:graphX}.


One of the central issues of causal effect estimation is identifying features that are confounders and controlling for them. Let $Y_t(X)$ denote the counterfactual outcome associated when treatment $t$ is applied as an intervention given $X$.  We consider the simpler case when the observed set of covariates $X$ is admissible or eligible to be used for adjustment. In other words, for any treatment $t$, $Y_t \perp T |X$, i.e. the counterfactual outcome associated with any treatment $t$ is independent of the treatment choice in the observational data given $X$. We are interested in the problem of estimating the average treatment effect between the pair of treatments given by $\mathbb{E}_X \left[ Y_t - Y_{t'} \right]$. This is denoted by $\mathrm{ATE}$. If $X$ is admissible this can be estimated from observational data. Inverse propensity weighing, standardization and doubly robust estimation are standard techniques used \cite{guo2020survey,imbens2009recent}.

If $X \in \mathbb{R}^p$ is high dimensional (large $p$), however, the number of samples required to estimate the treatment effects  accurately becomes too large to be practical in many applications.  In  practice, features in $X$ are designed to include as many factors as possible to capture all relevant confounders that are needed to satisfy the admissibility criterion. \cite{shpitser2012identification} showed that if we know the semi-Markovian causal model behind the observational data, then one can algorithmically identify if a given subset of $X$ is admissible or not (even if $X$ is not admissible). 

In our work, we focus on the case when $X$ is admissible but no detailed causal model is available. We study sufficient conditions for identifying if a subset $S \subset X$ is admissible given that $X$ is known to be admissible. A subset $S_1$ is \textit{c-equivalent} to another subset $S_2$ if $S_2$ being admissible implies $S_1$ being admissible and vice versa. In other words, both subsets can be used for adjustment and will yield the same $\mathrm{ATE}$ estimate. We rely on sufficient conditions for $c$-equivalence in \cite{pearl2009causality} as our main technical tool.

We consider a coarser causal model given in Figure \ref{fig:graph}, where $X$ has been decomposed into the sets $X_1$, $X_2$ (confounders), and $X_3$ (predictors) based on their connections to $Y$ and $T$. Applying sufficient conditions for $c$-equivalence, we show that it is sufficient to use either of two possible sets to form unbiased treatment effect estimates: $S = X_2 \cup X_3$ and $X_1 \cup X_2$.\footnote{Nodes in $X$ that do not have edges to either $T$ or $Y$ should not be included in either of the two sets. We omit these from the figure for simplicity.}

Prior work on sparse feature selection for treatment effect estimation has focused on the case where $X_1 \cup X_2$ is sparse \cite{shortreed2017outcome, cheng2020sufficient}, i.e. the number of confounding variables plus the number of variables biasing the treatment is small. In this work, we complete the picture by considering the companion setting where instead $S = X_2 \cup X_3$ is sparse, i.e. the number of confounding variables plus the number of predictors is small. In practice, we suggest running both our method and a method that identifies $X_1 \cup X_2$ and choosing the one that yields the lowest variance unbiased estimate. 
An added benefit of using $S$ over $X_1 \cup X_2$ is that $S$ includes the set of predictors, which serve to reduce the variance of the treatment effect estimate \cite{shortreed2017outcome}.


\textbf{Contributions:}
Given $X$ is admissible and given $q$ treatment cohorts, under a linear outcome model for $Y$ given $T$ and $X$, we show that maximizing least squares likelihood with a joint sparse non convex regularization recovers the subset $S$ of interest and the number of samples required is $O(kq \log p)$ and the error in the support recovery scales as $O ( \sqrt{\frac{\log p}{n}})$ where $|X|=p$ and $n$ is the number of samples. We demonstrate the effectiveness of our subset identification step in synthetic experiments as well in combination with doubly robust $\mathrm{ATE}$ estimation procedures on real datasets.




\begin{figure}[tb]
\begin{center}
\resizebox{.8\columnwidth}{!}{%
\begin{tikzpicture}
      \node[var] (Y) at (4,0) {$Y$};
      \node[var] (T) at (0,0) {$T$};
      \node[var] (X) at (2,1) {$X$};
      \draw[arr] (T) edge (Y);
      \draw[arr] (X) edge (T);
      \draw[arr] (X) edge (Y);
\end{tikzpicture}
}
\end{center}
\caption{Causal graph. $T$ is a discrete treatment, taking up to $q$ values, and $Y$ is a scalar outcome. $X$ is an observed set of $p$ covariates. }
\label{fig:graphX}
\end{figure}
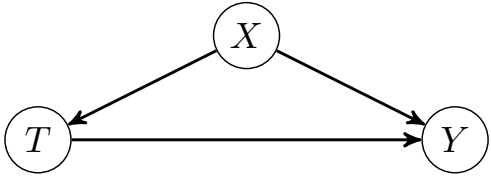

\textbf{Prior Work:}
\cite{guo2020survey,imbens2009recent} provide surveys of  methods that address causal effect estimation with observational data both from machine learning and econometrics perspectives. These surveys review classic approaches to $\mathrm{ATE}$ estimation including propensity weighing, doubly robust estimation and matching techniques. We only briefly review a small subset of these works in what follows. 

Perhaps the most relevant to our work is \cite{shortreed2017outcome} which also considered variable selection for causal inference in the regime stated in Figure \ref{fig:graphX}. In contrast to our approach that regresses $Y$ on $S$ given a fixed $T$, they use an outcome adaptive lasso\footnote{Lasso weighted by the unregularized coefficients.} sparse regression on the logistic transformation of $P(T=1|X)$, in order to find the $X_1 \cup X_2$ set. This choice limits the approach to binary treatments, and the associated theory is limited to asymptotic consistency, with no indication of sample complexity relative to the sparsity or dimensionality. \cite{cheng2020sufficient}, instead of finding a sparse subset $X_1 \cup X_2$, transform the covariates into a low dimensional space that satisfies conditional independence criteria.

A growing body of recent work has been applying machine learning to ITE estimation.
\cite{athey2015machine,kuenzel2019heterogeneous} introduce meta frameworks for applying supervised learning for ITE estimation. \cite{hill2011bayesian} applies Bayesian techniques to ITE estimation. Inspired by the rise of deep learning,
\cite{kallus2018deepmatch} used adversarial training to find covariate representations that match across treatment cohorts. Various recent works apply domain adaptation techniques to learn deep representations that match the treatment cohorts
\cite{yao2018representation,shalit2017estimating,yoon2018ganite}. \cite{louizos2017causal} uses variational autoencoders to find noisy proxies for latent confounders, and uses the result for ITE estimation. \cite{wager2018estimation} leverage latest advances in learning using forests for ITE estimation problems. When ITE/ATE is not identifiable from data, interval estimates on treatment effects have been obtained in \cite{kallus2019interval,yadlowsky2018bounds}.


 \textbf{Notation:} For a matrix $A \in \mathbb{R}^{p \times q}$, we define $A_{i:}$ to be the $i$th row of $A$ and $A_{:j}$ to be the $j$th column of $A$.
 We also define the norm $\|A\|_{a,b}$ for $a,b \in \mathbb{R}^+ \cup \infty$ as $\|A\|_{a,b}^a = \sum_{i =1}^p \|A_{i:}\|_b^a$. We denote $|||A|||_a$ as the $a$th order matrix norm, and $\|A\|_F = \|A\|_{2,2}$ as the Frobenius norm.




\begin{figure}[tb]
\begin{center}
\resizebox{\columnwidth}{!}{%
\begin{tikzpicture}
      \node[var] (Y) at (5.5,0) {$Y$};
      \node[var] (T) at (0,0) {$T$};
      \node[var] (S) at (2.75,1) {$X_2$};
      \node[var] (Xp) at (1.25,1) {$X_1$};
      \node[var] (X3) at (4.25,1) {$X_3$};
      \draw[arr] (X3) edge (Y);
      \draw[arr] (X3) edge (S);
      \draw[arr] (S) edge (X3);
      \draw[arr] (Xp) edge[bend left] (X3);
      \draw[arr] (X3) edge[bend right] (Xp);
      \draw[arr] (T) edge (Y);
      \draw[arr] (S) edge (T);
      \draw[arr] (S) edge (Y);
      \draw[arr] (Xp) edge (T);
      \draw[arr] (Xp) edge (S);
      \draw[arr] (S) edge (Xp);
      \draw[dashed] (.5,.5) rectangle (5,2);
      \node[text width = .7cm] at (.9,1.75){$X$};
      \draw[dashed] (2,.6) rectangle (4.9,1.9);
      \node[text width = .7cm] at (4.85,1.65){$S$};
\end{tikzpicture}
}
\end{center}
\caption{Partition of $X$ by connections to $T$ and $Y$. $X$ is composed of $X_1$ (arrows into $T$ not $Y$), $X_2$ (arrows into both $T$ and $S$, i.e. confounders), and $X_3$ (arrows into $Y$ not $T$, i.e. predictors). The identities of these sets are not known a priori and must be discovered from data. $S$ is composed of $X_2$ and $X_3$. }
\label{fig:graph}
\end{figure}
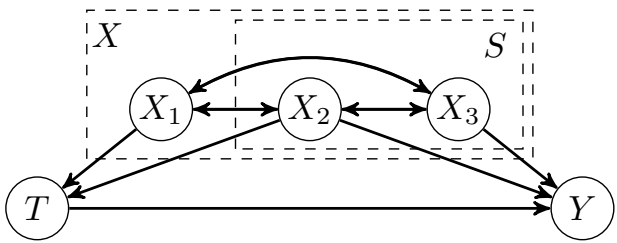


\section{Treatment Effects and Admissible Sets}

In the case of binary treatments, the average treatment effect is given by
\[
E[Y|\mathrm{do}(T = 1)] - E[Y|\mathrm{do}(T=0)],
\]
and the individual treatment effect by
\[
E[Y|X,\mathrm{do}(T = 1)] - E[Y|X,\mathrm{do}(T=0)].
\]
For higher cardinality $T$, similar pairwise differences involving the $E[Y|X, \mathrm{do}(T=t)]$ are in order.

Since we only have observational data, we make use the following property from \cite{pearl2009causality}:
\begin{definition}[Admissibility]
A set $X$ is called admissible if 
\begin{equation}\label{eq:admis}
p(y|\mathrm{do}(T=t)) = \int p(y|t, x) p(x) dx,
\end{equation}
i.e. we can compute the causal effect using observational probabilities controlling for $X$.
\end{definition}
$X$ will be admissible if there are no hidden confounders. Note that the dimensionality of the integral \eqref{eq:admis} is large since $X$ is high dimensional. Can we simplify this expression to involve only the sparse subset $S$? We make use of the following definition \cite{pearl2009causality}.
\begin{definition}[$c$-equivalence]\label{def:cequiv}
Two subsets $S_1$ and $S_2$ are $c$-equivalent if 
\[
\int p(y|t,s_1)p(s_1)d s_1 =  \int p(y|t,s_2)p(s_2)d s_2,
\]
i.e. the causal effect distributions controlling for $S_1$ and $S_2$ are equal.
\end{definition}
Definition \ref{def:cequiv} implies that if $S_2$ is $c$-equivalent to $S_1$ and $S_1$ is admissible, then so is $S_2$.

We now show $S$ and $X_1 \cup X_2$ are $c$-equivalent to $X$. 
\begin{lemma}\label{thm:admiss}
Given the causal graph in Figure \ref{fig:graph}, both the subset $S$ and the subset $X_1 \cup X_2$ are $c$-equivalent to the set $X$, hence either subset is sufficient as control to compute an unbiased estimate of the treatment effect.
\end{lemma}
This result is proved in supplement Section \ref{app:admiss}, and relies on two sufficient conditions for strong ignorability and $c$-equivalence given in Chapter 11 of \cite{pearl2009causality}.

Lemma \ref{thm:admiss} establishes that there is no additional bias resulting from controlling for either $S$ only or $X_1 \cup X_2$ only instead of the full $X$. Assuming that $X$ is an admissible set, i.e. there are no hidden confounders, there will be no bias and $S$ will also be an admissible set. The question then is which of these two sets to use as control. We suggest that when $X$ is high dimensional, in general the sparser of the two admissible subsets $S$ and $X_1 \cup X_2$ should be used (bearing in mind that all else being equal, effect estimates with $S$ will be lower variance since it includes all predictors $X_3$ of $Y$). Previous works such as \cite{shortreed2017outcome} focused on finding and controlling for $X_1 \cup X_2$, in this work we close the loop by proposing an estimator for the alternative admissible set $S$ and theoretically proving its lower sample complexity when $T$ is discrete and $Y$ continuous.

Given observational samples of $X, Y, T$, our goal is thus to find the smallest subset $S$ containing all nodes in $X$ that have an edge pointing towards $Y$ in the graph.\footnote{Note that if there are nodes in $X_3$ that do not have any direct connections to $X_1$, they are not needed for admissibility. We choose to include them in $S$ since they reduce the variance of the treatment effect estimator. } Since we have assumed that the outcome $Y$ does not have any edge pointing to $X$ or $T$, it is sufficient to use observational data to condition on $T=t$ and find the set of nodes $S_t$ in $X$ that have edges connecting to $Y$ in the undirected graph, and then take the union over $t$ as $S = \bigcup_{t=1}^q S_t$. 




\section{Oracle ATE/ITE}
In this section, we describe treatment effect estimation in the oracle setting where the support $S$ is known.
Suppose that an oracle gives us the identity of the optimal $S$ subset. By Lemma \ref{thm:admiss}, we have that
\[
P(Y| T) = \int P(Y|S,T) P(S) dS.
\]
This can be estimated directly from empirical probabilities, although with continuous $S$ the sample complexity is still significant without additional assumptions. In this work, we make use of the following assumption of linearity with respect to $S$ (to be relaxed in future work).
\begin{assumption}[Linearity]\label{ass:lin}
Assume that $Y$ follows the following generative model depending on $T$ and $S$:
\[
Y = \theta_{:t}^T S + \epsilon,
\]
where $\theta \in \mathbb{R}^{k \times q}$ is a matrix of linear coefficients and $\epsilon$ is i.i.d. noise. 
\end{assumption}


Suppose a regression estimate $\hat{\theta} \in \mathbb{R}^{k\times q}$ of the coefficient matrix $\theta$ is available.
For binary treatments, the individual treatment effect (ITE) can then be estimated by
\[
\widehat{ITE}(X) = (\hat{\theta}_{:1} - \hat{\theta}_{:0})^T S,
\]
where the $\hat{\theta}_i$ are regression coefficient estimates. 

Similarly, we can estimate the average treatment effect (ATE) as
\[
\widehat{ATE} = (\hat{\theta}_{:1} - \hat{\theta}_{:0})^T {\mu}_S,
\]
where ${\mu}_S$ is the specified mean of $S$.

We have the following lemma relating treatment effect estimation error to coefficient estimation error. The proof is immediate from norm inequalities.
\begin{lemma}[Oracle Effect Estimation Error]\label{lem:esterr}
Given Assumption \ref{ass:lin}, we have for binary treatments:
$|\widehat{ITE}(S) - {ITE(S)}| \leq \|S\|_1 \cdot\sum_{t=0}^1\|\hat{\theta}_{:t} - {\theta}_{:t}\|_\infty$, 
    $|\widehat{ATE} - {ATE}| \leq \|\mu_S\|_1 \cdot\sum_{t=0}^1\|\hat{\theta}_{:t} - {\theta}_{:t}\|_\infty$.
More generally, for $q$ treatments define $\tau(t) =E[Y|S, \mathrm{do}(T=t)] = \theta_{:t}^T S$. We have for all $t$ that
$|\hat{\tau}(t) - \tau(t)| \leq \|S\|_1 \|\hat{\theta} - \theta\|_{\infty,\infty}$.
\end{lemma}
Note that the error bounds in Lemma \ref{lem:esterr} grow linearly with $\|S\|_1$, which tends to grow linearly with the cardinality $|S| = k$. This confirms our motivation for finding sparse solutions to reduce sample complexity. 
In this section, we assumed that the sparse admissible set $S$ was given to us by an oracle. In the next section, we consider the real world setting where we must recover $S$ from the data itself.




\section{Jointly Sparse Variable Selection}



Our goal is to estimate the matrix of linear coefficients $\theta$ in Assumption \ref{ass:lin} using sparse regression of $Y$ versus $X$ given fixed $T$.
The classic approach to sparse regression is the lasso objective, which in our setting is
\[
\hat{\theta}_{:,j} = \arg\min_{\theta \in \mathbb{R}^p} \frac{1}{2}\theta^T \frac{X_j^T X_j}{n} \theta - \frac{y_j^T X_j}{n}\theta + \lambda \|\theta\|_1,
\]
where $X_j, y_j$ are samples from the $T = j$ conditional.\footnote{For simplicity, throughout the paper we assume $n$ samples are available from each conditional. The results can be easily adjusted to the case of imbalanced sample sets.}

Since we care about the union of the nonzero supports of the $\theta_{:j}$, it is wasteful to force an entry to zero in the $t=0$ graph if we know it is nonzero in $t=1$, etc. Hence, we instead use \emph{group sparsity}, which couples the sparsity of the $q$ vectors together. 

Traditionally, group sparsity is encouraged via the L-1,2 norm \cite{huang2010benefit,lounici2011oracle}, which is an L1 norm of the L2 norms of the rows of $\theta$. Copying the above, we can write the group-sparsity based objective as
\begin{align*}
\hat{\theta} &= \arg\min_{\theta\in\mathbb{R}^{p\times q}} \sum_{j=1}^q \left[\frac{1}{2}\theta_{:j}^T \frac{X_j^T X_j}{n} \theta_{:j} - \frac{y_j^T X_j}{n}\theta_{:j}\right] + \lambda \|\theta\|_{1,2}
\end{align*}
This can be solved iteratively as it is a convex problem, and 2-norm error bounds exist \cite{huang2010benefit,lounici2011oracle}.

Unfortunately, it is known that L1 based regression, while successful in estimating coefficients in terms of L2 norm error, does not perform well for variable selection without complex incoherence assumptions \cite{loh2017support}. 
To avoid these difficult-to-interpret assumptions, instead of L1 we will rely on the following class of nonconvex regularizers that retain the sparsity-promoting properties of the ``cusp" at $t=0$, while using a nonconvex shape to not excessively penalize large coefficients. 
\begin{definition}[$(\mu,\gamma)$-amenability]
\label{def:amen}
A regularization function $\rho_\lambda$ with parameter $\lambda$ is $\mu$-amenable for some $\mu >0$ if the following hold:
\begin{itemize}
 \vspace{-1mm}
    \item $\rho_\lambda$ is symmetric around 0 and $\rho_\lambda(0) = 0$.
    \vspace{-1mm}
    \item $\rho_\lambda$ is nondecreasing on $\mathbb{R}^+$.
     \vspace{-1mm}
    \item the function $\frac{\rho_\lambda(t)}{t}$ is nonincreasing on $\mathbb{R}^+$.
     \vspace{-1mm}
    \item $\rho_\lambda(t)$ is differentiable at all $t \neq 0$.
     \vspace{-1mm}
    \item $\rho_\lambda + \frac{\mu}{2} t^2$ is convex.
     \vspace{-1mm}
    \item $\lim_{t\rightarrow 0^+} \rho'_\lambda(t) = \lambda$.
     \vspace{-1mm}
\end{itemize}
If in addition there is some scalar $\gamma \in (0,\infty)$ such that $\rho'_\lambda = 0$ for all $t\geq \gamma \lambda$, then $\rho_\lambda$ is $(\mu,\gamma)$-amenable.

\end{definition}
Two example $(\mu,\gamma)$ amenable regularizers are the SCAD \cite{fan2001variable} and MCP \cite{zhang2010nearly} penalties. For convenience, define $q_\lambda(t) = \lambda |t| - \rho_\lambda(t)$. If $\rho_\lambda$ is $(\mu,\gamma)$ amenable, then $q_\lambda$ is everywhere differentiable.

Applying a $(\mu,\gamma)$ regularizer $\rho_\lambda$ on the row 2-norms we have the following objective function:
\begin{align}
\hat{\theta} =& \arg\min_{ \|\theta\|_{1,2} \leq R} \left\{\sum_{j=1}^q\left[ \frac{1}{2}\theta_{:j}^T \frac{X_j^T X_j}{n} \theta_{:j}  - \frac{y_j^T X_j}{n}\theta_{:j}\right]\right. \nonumber\\&\qquad\left.+ \sum_{i=1}^p  \rho_\lambda\left(\|\theta_{i:}\|_{2}\right)\right\},
\label{eq:nonconvex}
\end{align}
For convenience, define the unregularized loss function
\begin{equation}\label{eq:lossfunc}
    \mathcal{L}_n(\theta) = \sum_{j=1}^q\left[ \frac{1}{2}\theta_{:j}^T \frac{X_j^T X_j}{n} \theta_{:j}  - \frac{y_j^T X_j}{n}\theta_{:j}\right].
\end{equation}

We show below that the objective \eqref{eq:nonconvex} is a convex problem when $R$ and $\mu$ are chosen appropriately. We thus optimize the objective \ref{eq:nonconvex} using proximal gradient descent. Since the function $q_\lambda(t) = \lambda |t| - \rho_\lambda(t) $ is everywhere differentiable, it can be included in the gradient step computation, leaving the proximal step to be the proximal operator for the L-1,2 norm. This proximal operator is simply a soft thresholding on the norms of the rows of $\theta$, i.e. setting the rows of $\theta$ as $\max(0, \|\theta_{i:}\|_2 - \lambda)\frac{\theta_{i:}}{\|\theta_{i:}\|_2}$. If an optimization step would go outside the constraint set, we reduce the step size until the constraint is satisfied. A summary of the optimization algorithm is in supplement Section \ref{app:optim}. 






\subsection{Theoretical Analysis}







For the analysis, we make the additional assumption\footnote{We use the definition of subgaussianity from \cite{vershynin2010introduction}.}
\begin{assumption}[Subgaussianity]
Assume that conditioned on $T = t$, $X$ is subgaussian with parameter bounded from above by $\sigma_x$ for all $t$, and the noise term $\epsilon$ given in Asspt. \ref{ass:lin} is subgaussian with parameter $\sigma_\epsilon$.
\end{assumption}

We can then state the following bound.
\begin{theorem}\label{thm:consistency}
Suppose $p > q$ and for all $j = 1,\dots, q$, $y_j = (\theta^\ast_{:j})^T X_j + \epsilon$ where $X$ and $\epsilon$ are subgaussian, and all $\theta^\ast_{:j}$ have support contained in some unknown set $S$ with $|S| = k$ where $k$ is unknown. Furthermore, choose $(\lambda,R)$ such that $\|\theta^\ast\|_{1,2} < \frac{R}{2}$ and $c_\ell \sqrt{\frac{q\log p}{n}} \leq \lambda \leq \frac{c_u \sqrt{q}}{R}$, and assume $n \geq C \max\{R^2,k\}q \log p$ for some constants $c_\ell, c_u, C$ described in the proof. 
Suppose $\rho_\lambda$ is a $(\mu,\gamma)$-amenable regularizer with $\mu < \frac{1}{2}\min_j \lambda_{\min} (\Sigma^{(j)}_x)$ where $\Sigma^{(j)}_x = \mathbb{E}_{X_j}\frac{1}{n}X_j^T X_j$. Finally, suppose that 
\begin{equation}\label{eq:mintheta}
\theta^\ast_{\min} := \min_{i\in S} \|\theta^\ast_{i:}\|_2 \geq \lambda\gamma  + c_3 \sqrt{\frac{\log p}{n}}.
\end{equation}
Then with probability at least $1 - c_1 \exp(-c_2 \min[k, \log p])$ the objective \eqref{eq:nonconvex} has a unique stationary point $\hat{\theta}$ with support equal to $S$ and 
\[
\|\hat{\theta} - \theta^\ast\|_{\infty,\infty} \leq c_3 \sqrt{\frac{\log p}{n}}.
\]
\end{theorem}

\begin{remark}
This proof technique can also yield consistency for the L-1,2 norm regularizer with an appropriate incoherence assumption, see \cite{loh2017support} Proposition 3.
\end{remark}
\begin{remark}
The infinity norm error rate in Theorem \ref{thm:consistency} is optimal (since it coincides with the estimation error of the optimal oracle estimator).
\end{remark}
\begin{remark}[L2 error bounds]
The infinity norm bounds given in Theorem \ref{thm:consistency} yield tight L-$\infty,2$ and Frobenius norm bounds via standard norm inequalities.
\end{remark}


\begin{proof}[Proof of Theorem \ref{thm:consistency}]
Various steps in the proof are outlined below:
\begin{enumerate}
    \item[0.] Define and verify a joint Restricted Strong Convexity condition.
    \item[1.] Optimize the \emph{oracle} program where the supports of $\hat{\theta}_{:j}$ are restricted to the true $S$:
    \begin{align}\nonumber
\hat{\theta} = \arg\min_{\theta \in S, \|\theta\|_{1,2} \leq R} &\sum_{j=1}^q\left[ \frac{1}{2}\theta_{:j}^T \frac{X_j^T X_j}{n} \theta_{:j}  - \frac{y_j^T X_j}{n}\theta_{:j}\right] \\&+ \sum_{i\in S}  \rho_\lambda\left(\|\theta_{i:}\|_{2}\right),
    \label{eq:objfun}
    \end{align}
    and show the solution is in the interior of the constraint set. Under the restricted strong convexity assumption, this implies that the solution is a zero subgradient point.
    \item[2.] Define the dual variable $\hat z$ where $\hat{z}_S \in \nabla \|\hat{\theta}_S\|_{1,2}$ and $\hat{z}_{S^c}$ satisfying the zero subgradient condition, and establish \emph{strict dual feasibility} of $\hat{z}_{S^c}$ by showing that $\|\hat{z}_{S^c}\|_{\infty,2} \leq 1$. This implies $\hat{\theta}$ is a stationary point of the full objective \eqref{eq:nonconvex}.
    \item[3.] Show that $\hat{\theta}$ is the unique global minimum of the full objective \eqref{eq:nonconvex}.
\end{enumerate}

\paragraph{Step 0:} First, we verify a restricted strong convexity condition. Adapted from the $q=1$ case in \cite{loh2017support}, we require the following property of the loss function:
\begin{definition}[Joint Restricted Strong Convexity (Joint RSC)]
We say a loss $\mathcal{L}_n(\theta)$, $\theta \in \mathbb{R}^{p\times q}$ satisfies an ($\alpha,\tau$) joint RSC condition if for all $\Delta \in \mathbb{R}^{p\times q}$
\begin{align}
\langle \nabla \mathcal{L}_n&(\theta + \Delta) - \nabla \mathcal{L}_n(\theta), \Delta \rangle \label{eq:RSC}\\&\geq \left\{\begin{array}{ll} \alpha_1 \|\Delta\|_F^2 - \tau_1 \frac{\log p}{n} \|\Delta\|_{1,2}^2 & \|\Delta\|_F \leq 1\\ 
\alpha_2 \|\Delta\|_F - \tau_2\sqrt{ \frac{\log p}{n}} \|\Delta\|_{1,2} & \|\Delta\|_F \geq 1.
\end{array}\right.\nonumber
\end{align}
\end{definition}

The following is proven in supplement Section \ref{app:lemRSC}.
\begin{lemma}[Joint RSC for least squares loss]
Assume that $n \geq O(k \log p)$ and $n \geq 4 R^2 q \log p$. With high probability (at least $1 - q c_1 \exp(-cn)$), $\mathcal{L}_n$ is $(\alpha, \tau)$-joint RSC for $\alpha_1 = \alpha_2 = \frac{1}{2}\min_j(\lambda_{\min }(\Sigma_x^{(j)}))$ and $\tau_1 = q$, $\tau_2 = \sqrt{q}$. Furthermore, the objective \eqref{eq:objfun} is strongly convex on $\mathbb{R}^S$.
\label{lem:RSC}
\end{lemma}
We also have that with high probability
\begin{equation}\label{eq:nabla}
\|\nabla\mathcal{L}_n(\theta^\ast)\|_{\infty,2} \leq c' \sqrt{\frac{q \log p}{n}},
\end{equation}
by applying a norm inequality (2-norm is $\leq \sqrt{q}$ times infinity norm)  to the union bounded bound in the proof of Corollary 1 in \cite{loh2015regularized} (the $q=1$ case) and using $q < p$. 

\textbf{Step 1:} We recall $\|\theta^\ast\|_{1,2} \leq R/2$ and use the joint RSC conditions to bound $\|\tilde{\nu}\|_{1,2}$, where we set $\tilde{\nu} := \hat{\theta} - \theta^*$.  
We state the result as a lemma, proven in supplement Section \ref{app:stepOne}.
\begin{lemma}\label{lem:stepOne}
Suppose $\hat{\theta}$ is a zero subgradient point of the objective \eqref{eq:objfun} supported on $S$, i.e.
\begin{equation}
\nabla \mathcal{L}_n(\hat{\theta}_S) + \nabla \rho_\lambda(\hat{\theta}_S) = 0.
\end{equation}
Then $\|\tilde{\nu}\|_{1,2} < \frac{R}{2}$, yielding $\|\hat{\theta}\|_{1,2} < R$.
\end{lemma}
Since $\|\hat{\theta}\|_{1,2}$ is strictly less than $R$, $\hat\theta$ is in the interior of the constraint set, and thus has zero subgradient.


\textbf{Step 2:} 
Denote $\hat \Gamma^{(j)} = \frac{X_j^T X_j}{n}$, $ \hat \gamma^{(j)} = \frac{X_j^T y_j}{n}$. Then taking the gradients of \eqref{eq:objfun} yields for all $j$
\begin{equation}\label{eq:gradients}
\nabla \mathcal{L}_n (\theta_{:j}) = \hat{\Gamma}^{(j)} \theta_{:j} - \hat{\gamma}^{(j)}, \quad \nabla^2 \mathcal{L}_n (\theta_{:j}) = \hat{\Gamma}^{(j)}. 
\end{equation}
Consider the estimator $\hat{\theta}^{\mathcal{O}}$ formed by solving \eqref{eq:objfun} with $\lambda = 0$. 
We then can write
\[
\hat{\Gamma}^{(j)}(\hat\theta^{\mathcal{O}}_{:j} - \theta^\ast_{:j}) = \nabla \mathcal{L}_n(\hat{\theta}^{\mathcal{O}}_{:j}) - \nabla \mathcal{L}_n({\theta}^\ast_{:j}), \forall j,
\]
yielding (since $\hat{\Gamma}_{SS}^{(j)}$ is invertible since $n \geq k$ by assumption) 
\begin{equation}\label{eq:finalNorm}
\hat\theta^{\mathcal{O}}_{Sj} - \theta^\ast_{Sj} = (\hat{\Gamma}_{SS}^{(j)})^{-1} ( -(\hat{\Gamma}_{SS}^{(j)} \theta^\ast_{Sj} - \hat{\gamma}_S^{(j)}). 
\end{equation}
Appendix D.1.1 of \cite{loh2017support} showed that
\begin{equation}\label{eq:c3}
\left\|(\hat{\Gamma}_{SS}^{(j)})^{-1}(\hat{\Gamma}_{SS}^{(j)} \theta^\ast_{Sj} - \hat{\gamma}_S^{(j)})\right\|_\infty \leq \lambda_{\max}^{1/2}(\Sigma_{x}^{(j)}) \sigma_\epsilon \sqrt{\frac{2 \log p}{n}},
\end{equation}
with probability at least $1 - c''_1 \exp(-c''_2\min(k,\log p))$.


Hence we obtain via the union bound that 
\begin{equation}\label{eq:prenormbd}
\|\hat\theta^{\mathcal{O}} - \theta^\ast\|_{\infty,\infty}  \leq c_3 \sqrt{\frac{ \log p}{n}}, \: 
\|\hat\theta^{\mathcal{O}} - \theta^\ast\|_{\infty,2}  \leq c_3 \sqrt{\frac{q \log p}{n}} 
\end{equation}
with probability at least $1 - c_1 \exp(-c_2\min(k,\log p))$ (since $k > \log q $ and $p > q$) where $c_1,c_2,c_3$ are constants. 


Now we have the following result, proved in supplement Section \ref{app:lem5}.
\begin{lemma}\label{lem:5}

Suppose $\rho_\lambda$ is $(\mu,\gamma)$ amenable and 
\[
\theta^\ast_{\min} = \min_{i \in S} \|\theta^\ast_{i:}\|_2 \geq \lambda\gamma  + c_3 \sqrt{\frac{\log p}{n}}.
\]
Then with probability at least $1 - c_1 \exp(-c_2\min(k,\log p))$
\[
\lambda \hat{z}_{i:} - \nabla q_\lambda (\|\hat{\theta}_{i:}\|_2) = 0 \quad \forall i \in S.
\]
\end{lemma}
Lemma \ref{lem:5} implies that if $\theta^\ast_{\min}$ satisfies the given condition, then $\nabla_{\theta_{S}}\rho_\lambda(\hat{\theta}_{S:}) = 0$, implying that $\hat\theta^{\mathcal{O}}$ is a zero subgradient point of \eqref{eq:objfun} and hence $\hat \theta = \hat\theta^{\mathcal{O}}$. Hence the bound \eqref{eq:prenormbd} also applies to $\hat{\theta}$ as in the theorem statement.


Now, define the shifted objective function as 
\begin{equation}\label{eq:shifted}
\bar{\mathcal{L}}_n(\theta) = \mathcal{L}_n(\theta) - \sum\nolimits_{i=1}^p q_\lambda(\|\theta_{i:}\|_2). 
\end{equation}
Making $\hat{\theta} = (\hat{\theta}_S, 0)$, the zero subgradient condition becomes
\begin{equation}\label{eq:zerosub}
\nabla\bar{\mathcal{L}}_n(\hat{\theta}) + \lambda \hat{z} = 0,
\end{equation}
where $\hat{z} \in \partial \|\hat \theta\|_{1,2}$.
Note that where rows of $\hat{\theta}$ are zero, the corresponding rows of $\hat{z}$ can be any vector in the unit 2-sphere. Where the rows are nonzero, it is a unit vector parallel to the row. Hence we have the strict dual feasibility condition $\|\hat{z}_c\|_{\infty,2} \leq 1 - \delta$ for some delta we choose later. 

We expand the zero subgradient condition \eqref{eq:zerosub} as
\begin{align}\label{eq:zerosub2}
&\left(\nabla{\mathcal{L}_n}(\hat{\theta}_{i:}) -  \nabla{\mathcal{L}_n}(\theta^\ast_{i:})\right) \\\nonumber&+ \left(\nabla{\mathcal{L}_n}(\theta^\ast_{i:}) - \nabla q_\lambda (\|\hat{\theta}_{i:}\|_2)\right) + \lambda \hat{z}_{i:} = 0,\quad \forall i.
\end{align}
Note that by the selection property, for all $i \notin S$, $\nabla q_\lambda (\|\hat{\theta}_{i:}\|_2) =  \nabla q_\lambda(0) = 0$. Additionally, by Lemma \ref{lem:5} combined with \eqref{eq:prenormbd} and the assumption \eqref{eq:mintheta} we know that $\lambda \hat{z}_{i:} - \nabla q_\lambda (\|\hat{\theta}_{i:}\|_2) = 0$ for all $i \in S$. 

Using \eqref{eq:gradients} we can then simplify the condition \eqref{eq:zerosub2} as
\begin{align}\label{eq:zerosub3}
\hat{\Gamma}^{(j)} (\hat{\theta}_{:j} \!- \!\theta^\ast_{:j} ) \!+\! \hat{\Gamma}^{(j)} \theta^\ast_{:j} \!-\! \hat{\gamma}^{(j)} + \left[\!\!\!\!\begin{array}{c} 0 \\ (\hat{z}_{S^c})_{:j} \end{array}\!\!\!\!\right] = 0,  \forall j.
\end{align}
Since furthermore we have $\hat{\theta}_{S^c} = \theta^\ast_{S^c} = 0$, this allows us to solve for each $[\hat{z}_{S^c}]_{:j}$ separately: 
\[
[\hat{z}_{S^c}]_{:j} = \frac{1}{\lambda}\left[ \hat{\gamma}^{(j)}_{S^c} - \hat{\Gamma}^{(j)}_{S^cS} [\hat{\Gamma}^{(j)}_SS]^{-1} \hat{\gamma}^{(j)}_S \right]
\]
where we have partitioned $\hat{\Gamma}^{(j)} = { \left[\begin{array}{cc} \scriptstyle\hat{\Gamma}^{(j)}_{SS} & \scriptstyle\hat{\Gamma}^{(j)}_{SS^c} \\\scriptstyle\hat{\Gamma}^{(j)}_{S^cS} & \scriptstyle\hat{\Gamma}^{(j)}_{S^cS^c} \end{array}\right]}$.

This quantity was analyzed by \cite{loh2017support} Appendix D.1.1. With probability at least $1 - c \exp(-c' \log p)$, 
\[
\| \hat{\gamma}^{(j)}_{S^c} - \hat{\Gamma}^{(j)}_{S^cS} [\hat{\Gamma}^{(j)}_{SS}]^{-1} \hat{\gamma}^{(j)}_S\|_{\infty} \leq C \sqrt{\frac{\log p}{n}},
\]
assuming $n \geq O(k \log p)$. Using the union bound and definition of $\infty,2$ norm, we then have that with probability at least $1 - c \exp(\log q-c' \log p)$
\begin{equation}\label{eq:step2}
\|\hat{z}_{S^c}\|_{\infty,2} \leq C \sqrt{\frac{q\log p}{n}}.
\end{equation}
Strict dual feasibility follows whenever $\lambda > C \sqrt{\frac{q\log p}{n}}$.




\textbf{Step 3:} Since by Step 2 $\hat{\theta}$ is a zero subgradient point of the full objective \eqref{eq:nonconvex}, it is also a local optima of the full objective \eqref{eq:nonconvex}. Furthermore, Lemma \ref{lem:3}, proven in supplement Section \ref{app:lem3} shows all local optima of \eqref{eq:nonconvex} must be supported on $S$.


\begin{lemma}\label{lem:3}
Suppose that $\tilde \theta$ is a stationary point of \eqref{eq:nonconvex} with $\|\hat{z}_{S^c}\|_{\infty,2} \leq 1/2$ and the conditions of Theorem \ref{thm:consistency} hold with $c_u = \frac{\alpha_2}{8}$ and $c_\ell = \sqrt{q^{-1} \tau_1 \alpha_2}$, and $n \geq \max\{\frac{16}{\alpha^2_2} R^2 \tau_2^2,  \frac{200\tau_1}{\alpha_1 - \mu} k \} \log p $. Then for all $j$, $\mathrm{supp}(\tilde \theta_{:j}) \subseteq S$.
\end{lemma}
Recall that in Step 2 \eqref{eq:step2} we showed that the condition of Lemma \ref{lem:3} is satisfied when $C \sqrt{\frac{q\log p}{n}} \leq 1/2$, i.e. whenever $n \geq 4C^2 q \log p$.
Hence by strict convexity on the $\mathbb{R}^S$ space (Lemma \ref{lem:RSC}), $\hat{\theta}_S$ is the unique global optimum of the full objective \eqref{eq:nonconvex}.
\end{proof}


\subsection{Implication for Effect Estimation}
Since we recover the support $S$ with high probability, we can plug in the bound for the oracle estimate and obtain the following bounds for the linear estimates.

\begin{lemma}[Effect Estimation Error]\label{lem:esterrSparse}
Given the assumptions of Theorem \ref{thm:consistency}, with high probability the following hold for our estimator. For binary treatments ($C$ is a constant):
$|\widehat{ITE}(S) - {ITE(S)}| \leq 2C\|S\|_1 \sqrt{\frac{\log p}{n}}$, 
    $|\widehat{ATE} - {ATE}| \leq 2C\|\mu_S\|_1 \sqrt{\frac{\log p}{n}}$.
More generally, for $q$ possible treatments define $\tau(t) =E[Y|S, \mathrm{do}(T=t)] = \theta_{:t}^T S$. We have for all $t$
$|\hat{\tau}(t) - \tau(t)| \leq C' \|S\|_1 \sqrt{\frac{q\log p}{n}}$.
\end{lemma}

\begin{remark}[Comparisons]
Note we only lose a log factor relative to the oracle estimator. The comparison to a nonsparse estimator (one that sets $S = X$) depends on $\|S\|_1$, but for diffuse $X$ such that subsets $S$ typically have $\|S\|_1 = O_p(|S|)$, our estimator improves on the nonsparse estimator by a factor of $\frac{|S| \sqrt{\log p}}{p}$ which is significant for sparse $S$. 
\end{remark}

\begin{remark}[Application to nonlinear settings]
We note that our $S$ recovery algorithm is not limited to being used in conjunction with linear effect estimation. Our approach can be used to find a sparse $S$, and then any desired effect estimator can be applied to the data, controlling only for the set $S$.
\end{remark}

\section{Experimental Results}
\subsection{Synthetic Data}
We use synthetic data generated as follows. $X$ is generated from an isotropic Gaussian distribution. $T$ is generated by sampling from a multinomial distribution with probabilities given by $\mathrm{softmax}(\Phi^T X)$, where $\Phi\in \mathbb{R}^{p\times q}$ has i.i.d. Gaussian elements. The output $Y$ is then generated according to the linear model in Assumption \ref{ass:lin}, where the $k$ nonzero rows of $\theta$ have been sampled from an i.i.d Gaussian distribution. We choose $k =10$, 
and use the MCP penalty \cite{zhang2010nearly} as our nonconvex regularizer $\rho$. 
For binary treatments, i.e. $q=2$, Figure \ref{fig:binary} shows the probability of correctly recovering the set $S$ (with cardinality 10) as the total size $p$ of $X$ and the number of samples $n$ are varied. Note that the number of samples required for consistent recovery of $S$ depends approximately logarithmically on $p$, as predicted. 
\begin{figure}[!t]
	\centering
		\includegraphics[width=.9\linewidth]{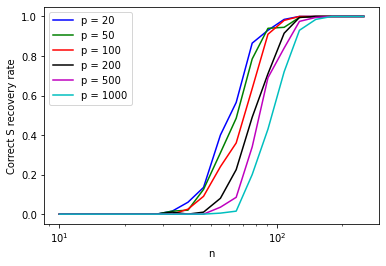}
    \caption{Empirical probability of our algorithm correctly recovering the sparse set $S$ as a function of $n$ and $p$, for binary actions ($q=2$).}
    \label{fig:binary}
\end{figure}
We next verify the benefits of using joint sparsity over a simple taking of the union of sparse subsets recovered independently for each value of $T$. Figure \ref{fig:multiary} compares our approach with the independent sparsity approach (also using nonconvex regularization) for $q=10$. Note that our algorithm significantly outperforms the independent sparsity approach.
Figure \ref{fig:multiary40} in the supplement shows results for $q=40$, showing that as $q$ increases, the number of samples required in fact decreases slightly (since in our experimental setup $\|\theta\|_{1,2}$ grows in expectation as $\sqrt{q}$). 
\begin{figure}[!t]
	\centering
		\includegraphics[width=\linewidth]{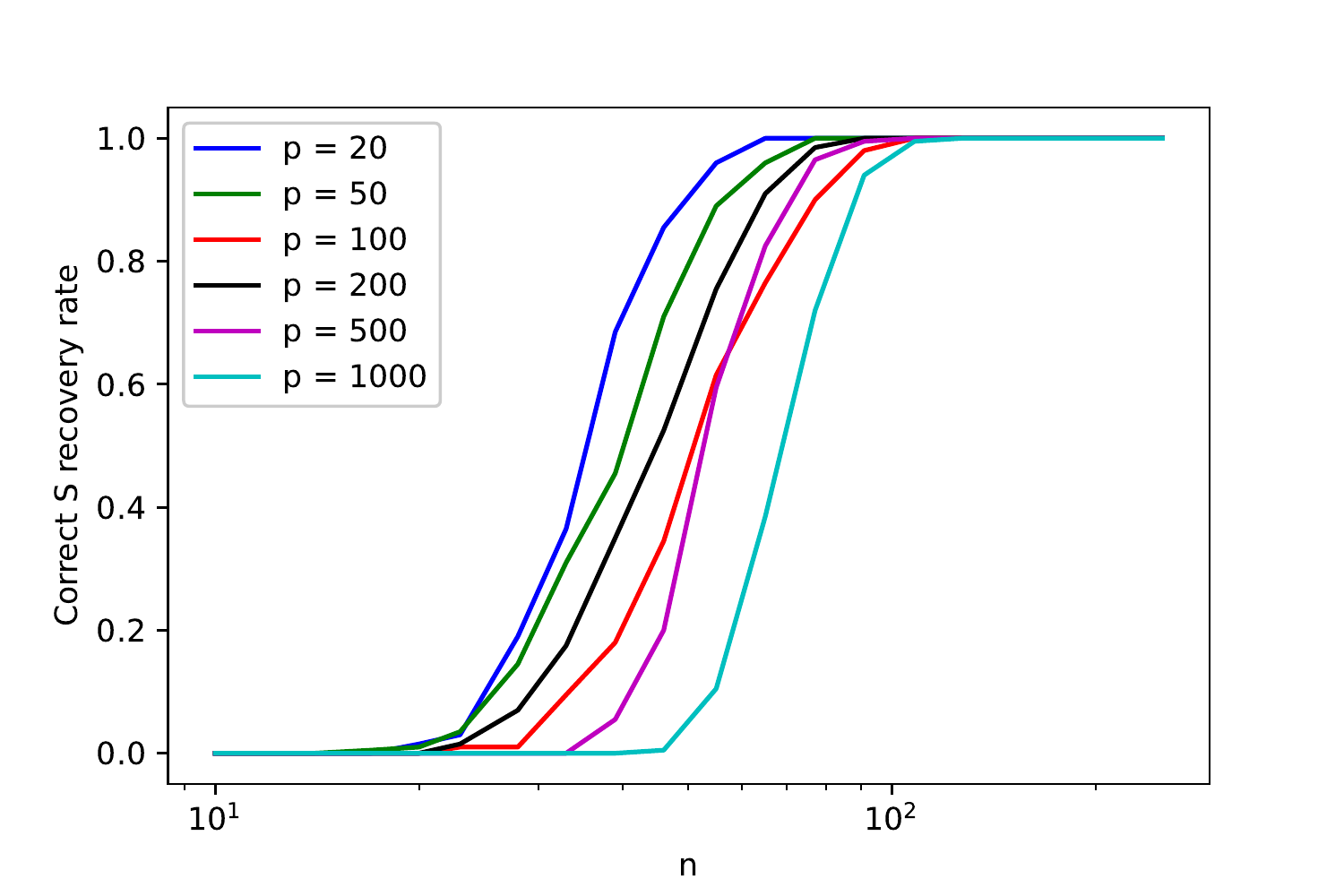}
		\vspace{-5mm}
		\includegraphics[width=\linewidth]{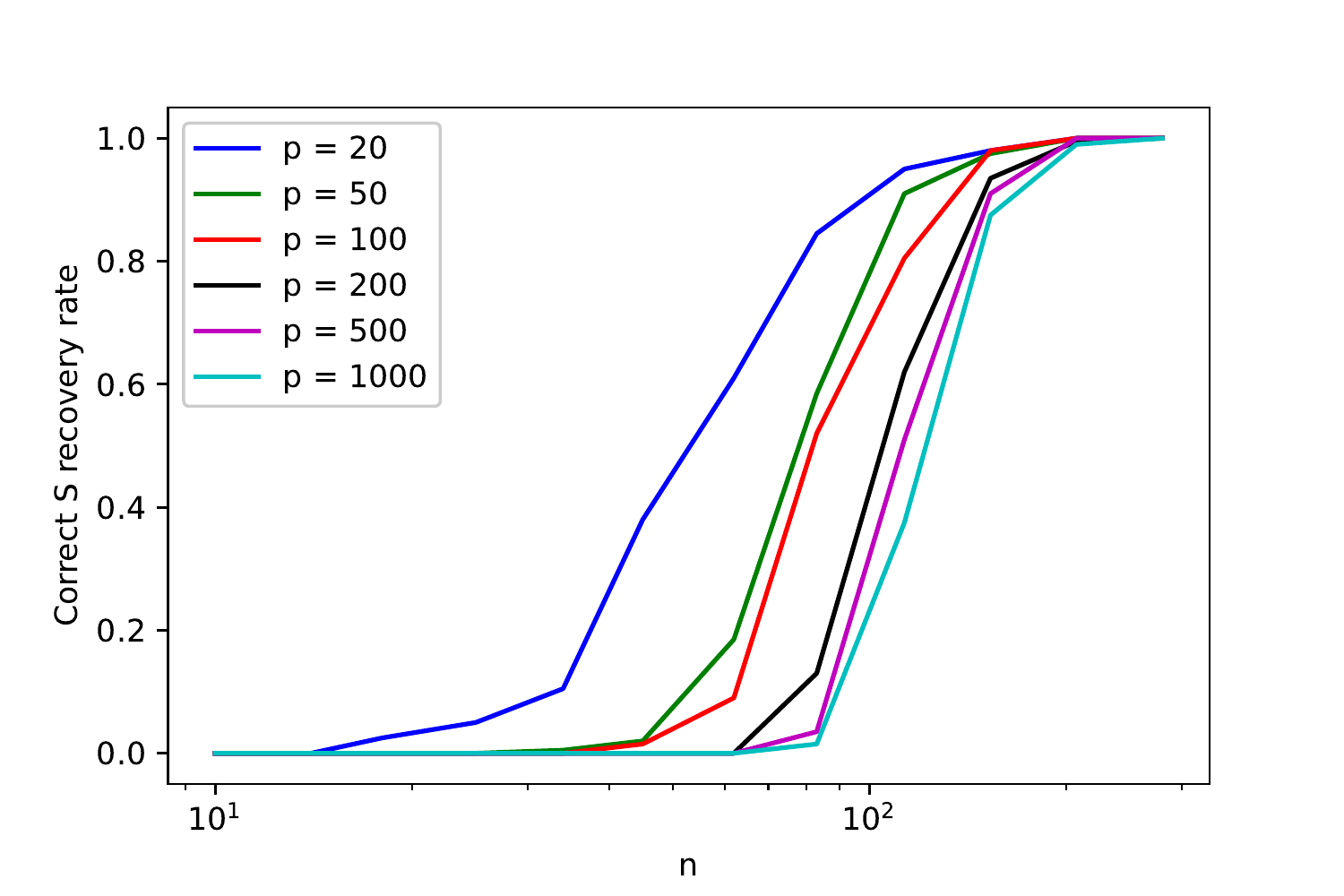}
    \caption{Empirical probability of our joint sparse algorithm (upper) and independent sparsity approach (lower) correctly recovering the sparse set $S$ as a function of $n$ and $p$, for $q=10$. }
    \label{fig:multiary}
\end{figure}


\subsection{Real Datasets}
\textbf{Cattaneo2: effect of smoking on birth weight.} This dataset\footnote{http://www.stata-press.com/data/r13/cattaneo2.dta} \cite{abadie2006large} studies the effect of maternal smoking on babies' birth weight in grams, and consists of 4642 singleton births in Pennsylvania, US. Actions are 0: no smoking (3778 samples), 1: 1-5 cigarettes daily (200 samples), 2: 6-10 cigarettes daily (337 samples), and 3: 11 or more cigarettes daily (327 samples). 20 covariates are included. Results comparing (nonsparse) doubly robust effect estimates \cite{causalevaluations} and the effect estimates obtained by using the doubly robust estimator on the sparse set $S$ obtained by our method are shown in Table \ref{tab:cattaneo}. We randomly split the data to have 20\% used for the $S$ estimation and 80\% used for the effect estimation. Our sparse approach is tuned via cross validation and on average yields a sparse $S$ of cardinality 10.9 (out of $|X| = 20$). Note that the sparse approach, unlike the full approach, yields a binary effect estimate consistent within the known empirical estimated interval \cite{abadie2006large}. For additional method comparisons for the binary effect, see Figure 5 in \cite{cheng2020sufficient} - only the dimensionality reduction method of \cite{cheng2020sufficient} provides an estimate in the empirical interval as we do.  

\begin{table}[htb]
\footnotesize
\centering
\begin{tabular}{|p{33mm}|p{20mm}|p{20mm}|}
\hline
 & {Nonsparse Doubly Robust Estimate} & {Sparse DR Estimate (Ours)}  \\\hline
Effect of 1 vs. 0 & -151.4g(21.3) & -195.0g(28.6)  \\\hline
Effect of 2 vs. 0 & -161.9g(16.6) & -264.0g(34.2) \\\hline
Effect of 3 vs. 0 & -189.2g(21.1) & -236.0g(26.6) \\\hline 
Binary effect($>0$ vs 0) & -162.4g(8.5) & -239.3g(10.8) \\\hline \end{tabular}
\caption{Estimated average treatment effects on Cattaneo2 dataset. Actions -- 0: no smoking, 1: 1-5 cigarettes daily, 2: 6-10 daily, and 3: 11 or more. For binary action effect, the empirical estimated interval is known to be  (-250g, -200g). Standard deviations over 20 random data splits are given in parentheses. }
\label{tab:cattaneo}
\end{table}

\textbf{IDHP: Effects of high-intensity care on low birth rate and premature infants.}
This semi-synthetic dataset\footnote{https://github.com/vdorie/npci} \cite{hill2011bayesian} consists of data on 25 covariates and an assigned treatment variable indicating whether the child was assigned to high-intensity care. Following the procedure in \cite{hill2011bayesian}, the treated and non-treated populations are biased and a response variable is generated according to the ``A'' scheme therein (which is designed to have sparse edges from $X$ to $Y$). Since the response is generated synthetically, the true ATE is known to be 4.36. Results for the non-sparse doubly robust estimator and the doubly robust estimator applied to the sparse $S$ recovered by our approach are shown in Table \ref{tab:idhp}. Both methods work reasonably well, with our sparse version outperforming (selecting on average $|S| =6.4$). The sparse performance is on par with the well-performing methods with results shown in Figure 2 of \cite{cheng2020sufficient}. 
\begin{table}[htb]
\footnotesize
\centering
\begin{tabular}{|p{6mm}|p{26mm}|p{30mm}|}
\hline
 & {Nonsparse Doubly \newline Robust Estimate} & {Sparse DR Estimate (Ours)}  \\\hline
ATE & 4.76(.51) & 4.49(.57)  \\\hline
\end{tabular}
\caption{Estimated average treatment effects on semi-synthetic IDHP dataset, with standard deviations over 20 random trials in parentheses. The true ATE is 4.36. }
\label{tab:idhp}
\end{table}

Additional details and real data experiments are in the supplement.

\section{Conclusion}
We considered using sparse regression to reduce the sample complexity of estimating causal effects in the presence of large numbers of covariates. We presented an algorithm based on joint-sparsity promoting nonconvex regularization, proved that it correctly recovers the sparse support $S$ with high probability, and tested it experimentally. 
In future work, we plan to use the power of the joint RSC concept to generalize our sparse estimator to more flexible nonlinear settings and to losses for categorical outcomes.

\clearpage
\bibliographystyle{apalike}
\bibliography{references}

\begin{thebibliography}{}

\bibitem[Abadie and Imbens, 2006]{abadie2006large}
Abadie, A. and Imbens, G.~W. (2006).
\newblock Large sample properties of matching estimators for average treatment
  effects.
\newblock {\em econometrica}, 74(1):235--267.

\bibitem[Athey and Imbens, 2015]{athey2015machine}
Athey, S. and Imbens, G.~W. (2015).
\newblock Machine learning methods for estimating heterogeneous causal effects.
\newblock {\em stat}, 1050(5):1--26.

\bibitem[Cheng et~al., 2020]{cheng2020sufficient}
Cheng, D., Li, J., Liu, L., and Liu, J. (2020).
\newblock Sufficient dimension reduction for average causal effect estimation.
\newblock {\em arXiv preprint arXiv:2009.06444}.

\bibitem[Fan and Li, 2001]{fan2001variable}
Fan, J. and Li, R. (2001).
\newblock Variable selection via nonconcave penalized likelihood and its oracle
  properties.
\newblock {\em Journal of the American statistical Association},
  96(456):1348--1360.

\bibitem[Guo et~al., 2020]{guo2020survey}
Guo, R., Cheng, L., Li, J., Hahn, P.~R., and Liu, H. (2020).
\newblock A survey of learning causality with data: Problems and methods.
\newblock {\em ACM Computing Surveys (CSUR)}, 53(4):1--37.

\bibitem[Hill, 2011]{hill2011bayesian}
Hill, J.~L. (2011).
\newblock Bayesian nonparametric modeling for causal inference.
\newblock {\em Journal of Computational and Graphical Statistics},
  20(1):217--240.

\bibitem[Huang et~al., 2010]{huang2010benefit}
Huang, J., Zhang, T., et~al. (2010).
\newblock The benefit of group sparsity.
\newblock {\em The Annals of Statistics}, 38(4):1978--2004.

\bibitem[Imbens and Wooldridge, 2009]{imbens2009recent}
Imbens, G.~W. and Wooldridge, J.~M. (2009).
\newblock Recent developments in the econometrics of program evaluation.
\newblock {\em Journal of economic literature}, 47(1):5--86.

\bibitem[Kallus, 2018]{kallus2018deepmatch}
Kallus, N. (2018).
\newblock Deepmatch: Balancing deep covariate representations for causal
  inference using adversarial training.
\newblock {\em arXiv preprint arXiv:1802.05664}.

\bibitem[Kallus et~al., 2019]{kallus2019interval}
Kallus, N., Mao, X., and Zhou, A. (2019).
\newblock Interval estimation of individual-level causal effects under
  unobserved confounding.
\newblock In {\em The 22nd International Conference on Artificial Intelligence
  and Statistics}, pages 2281--2290.

\bibitem[Kuenzel, 2019]{kuenzel2019heterogeneous}
Kuenzel, S.~R. (2019).
\newblock {\em Heterogeneous Treatment Effect Estimation Using Machine
  Learning}.
\newblock PhD thesis, UC Berkeley.

\bibitem[Loh and Wainwright, 2015]{loh2015regularized}
Loh, P.-L. and Wainwright, M.~J. (2015).
\newblock Regularized {M}-estimators with nonconvexity: Statistical and
  algorithmic theory for local optima.
\newblock {\em The Journal of Machine Learning Research}, 16(1):559--616.

\bibitem[Loh and Wainwright, 2017]{loh2017support}
Loh, P.-L. and Wainwright, M.~J. (2017).
\newblock Support recovery without incoherence: A case for nonconvex
  regularization.
\newblock {\em The Annals of Statistics}, 45(6):2455--2482.

\bibitem[Louizos et~al., 2017]{louizos2017causal}
Louizos, C., Shalit, U., Mooij, J.~M., Sontag, D., Zemel, R., and Welling, M.
  (2017).
\newblock Causal effect inference with deep latent-variable models.
\newblock In {\em Advances in Neural Information Processing Systems}, pages
  6446--6456.

\bibitem[Lounici et~al., 2011]{lounici2011oracle}
Lounici, K., Pontil, M., Van De~Geer, S., Tsybakov, A.~B., et~al. (2011).
\newblock Oracle inequalities and optimal inference under group sparsity.
\newblock {\em The annals of statistics}, 39(4):2164--2204.

\bibitem[Pearl, 2009]{pearl2009causality}
Pearl, J. (2009).
\newblock {\em Causality}.
\newblock Cambridge university press.

\bibitem[Shalit et~al., 2017]{shalit2017estimating}
Shalit, U., Johansson, F.~D., and Sontag, D. (2017).
\newblock Estimating individual treatment effect: generalization bounds and
  algorithms.
\newblock In {\em International Conference on Machine Learning}, pages
  3076--3085. PMLR.

\bibitem[Shimoni et~al., 2019]{causalevaluations}
Shimoni, Y., Karavani, E., Ravid, S., Bak, P., Ng, T.~H., Alford, S.~H., Meade,
  D., and Goldschmidt, Y. (2019).
\newblock An evaluation toolkit to guide model selection and cohort definition
  in causal inference.
\newblock {\em arXiv preprint arXiv:1906.00442}.

\bibitem[Shortreed and Ertefaie, 2017]{shortreed2017outcome}
Shortreed, S.~M. and Ertefaie, A. (2017).
\newblock Outcome-adaptive lasso: Variable selection for causal inference.
\newblock {\em Biometrics}, 73(4):1111--1122.

\bibitem[Shpitser and Pearl, 2012]{shpitser2012identification}
Shpitser, I. and Pearl, J. (2012).
\newblock Identification of conditional interventional distributions.
\newblock {\em arXiv preprint arXiv:1206.6876}.

\bibitem[Vershynin, 2010]{vershynin2010introduction}
Vershynin, R. (2010).
\newblock Introduction to the non-asymptotic analysis of random matrices.
\newblock {\em arXiv preprint arXiv:1011.3027}.

\bibitem[Vershynin, 2012]{vershynin2012close}
Vershynin, R. (2012).
\newblock How close is the sample covariance matrix to the actual covariance
  matrix?
\newblock {\em Journal of Theoretical Probability}, 25(3):655--686.

\bibitem[Wager and Athey, 2018]{wager2018estimation}
Wager, S. and Athey, S. (2018).
\newblock Estimation and inference of heterogeneous treatment effects using
  random forests.
\newblock {\em Journal of the American Statistical Association},
  113(523):1228--1242.

\bibitem[Yadlowsky et~al., 2018]{yadlowsky2018bounds}
Yadlowsky, S., Namkoong, H., Basu, S., Duchi, J., and Tian, L. (2018).
\newblock Bounds on the conditional and average treatment effect with
  unobserved confounding factors.
\newblock {\em arXiv preprint arXiv:1808.09521}.

\bibitem[Yao et~al., 2018]{yao2018representation}
Yao, L., Li, S., Li, Y., Huai, M., Gao, J., and Zhang, A. (2018).
\newblock Representation learning for treatment effect estimation from
  observational data.
\newblock In {\em Advances in Neural Information Processing Systems}, pages
  2633--2643.

\bibitem[Yoon et~al., 2018]{yoon2018ganite}
Yoon, J., Jordon, J., and van~der Schaar, M. (2018).
\newblock Ganite: Estimation of individualized treatment effects using
  generative adversarial nets.
\newblock In {\em International Conference on Learning Representations}.

\bibitem[Zhang et~al., 2010]{zhang2010nearly}
Zhang, C.-H. et~al. (2010).
\newblock Nearly unbiased variable selection under minimax concave penalty.
\newblock {\em The Annals of statistics}, 38(2):894--942.

\end{thebibliography}

\clearpage
\onecolumn
\hsize\textwidth
  \linewidth\hsize \toptitlebar {\centering
  {\Large\bfseries Supplementary Materials for: High-Dimensional Feature Selection for Sample Efficient Treatment Effect Estimation \par}}
 \bottomtitlebar \vskip 0.2in 

\section{Optimization algorithm}\label{app:optim}
The proximal gradient algorithm for optimizing our objective \eqref{eq:nonconvex} is shown in Algorithm \ref{alg:algor}, where we define (for $\theta \in \mathbb{R}^{p\times q}$)
\[
[\mathrm{Prox}_\lambda(\theta)]_{i:} = \theta_{i:} \max\left(0, 1 - \frac{\lambda}{\|\theta_{i:}\|_2}\right) ,\qquad i = 1,\dots, p,
\]
as the proximal operator for the L-1,2 norm.


\begin{algorithm}[h]
\caption{Proximal gradient descent for \eqref{eq:nonconvex}}
\label{alg:algor}
\begin{algorithmic}[1]
\STATE Input: matrices $\hat{\Gamma}^{(j)} = \frac{X_j^T X_j}{n}$, $\hat{\gamma}^{(j)} = \frac{X_j^T y_j}{n}$ for $j = 1,\dots, q$, regularizer $\rho_\lambda$ and associated $q'_{\lambda}(\cdot)$, backtracking constant $c \in (0,1)$, initial step size $\zeta_{0}$, norm constraint $R$, and initial iterate $\theta_0$.
\STATE $\theta \gets \theta_0$.
\WHILE{ not converged}

\FOR{$j = 1,\dots, q$}	
\STATE Compute the $j$th gradient $\nabla \bar{\mathcal{L}}_n(\theta_{:j}) = \hat{\Gamma}^{(j)} \theta_{:j} - \hat{\gamma}^{(j)} - \sum_{i=1}^p  \theta_{ij} \frac{q'_\lambda(\|\theta_{i:}\|_2)}{\|\theta_{i:}\|_2}$. 
\ENDFOR
\STATE \emph{Line search}: Let stepsize $\zeta_t$ be the largest element of $\{c^t \zeta_{0}\}_{t =1,\dots}$ such that 
\[
\|\mathrm{Prox}_{\lambda}(\theta - \zeta_t \nabla \bar{\mathcal{L}}_n(\theta)))\|_{1,2} < R.
\]
\STATE $\theta \gets \mathrm{Prox}_{\lambda}(\theta - \zeta_t \nabla \bar{\mathcal{L}}_n(\theta)))$.
\ENDWHILE
\STATE Return estimate $\theta$.
\end{algorithmic}
\end{algorithm}

\section{Additional Experiments}
\subsection{Synthetic experiments for $q=40$}
Figure \ref{fig:multiary40} shows results for $q=40$ following the setup in the main text.
\begin{figure}[ht]
	\centering
		\includegraphics[width=4in]{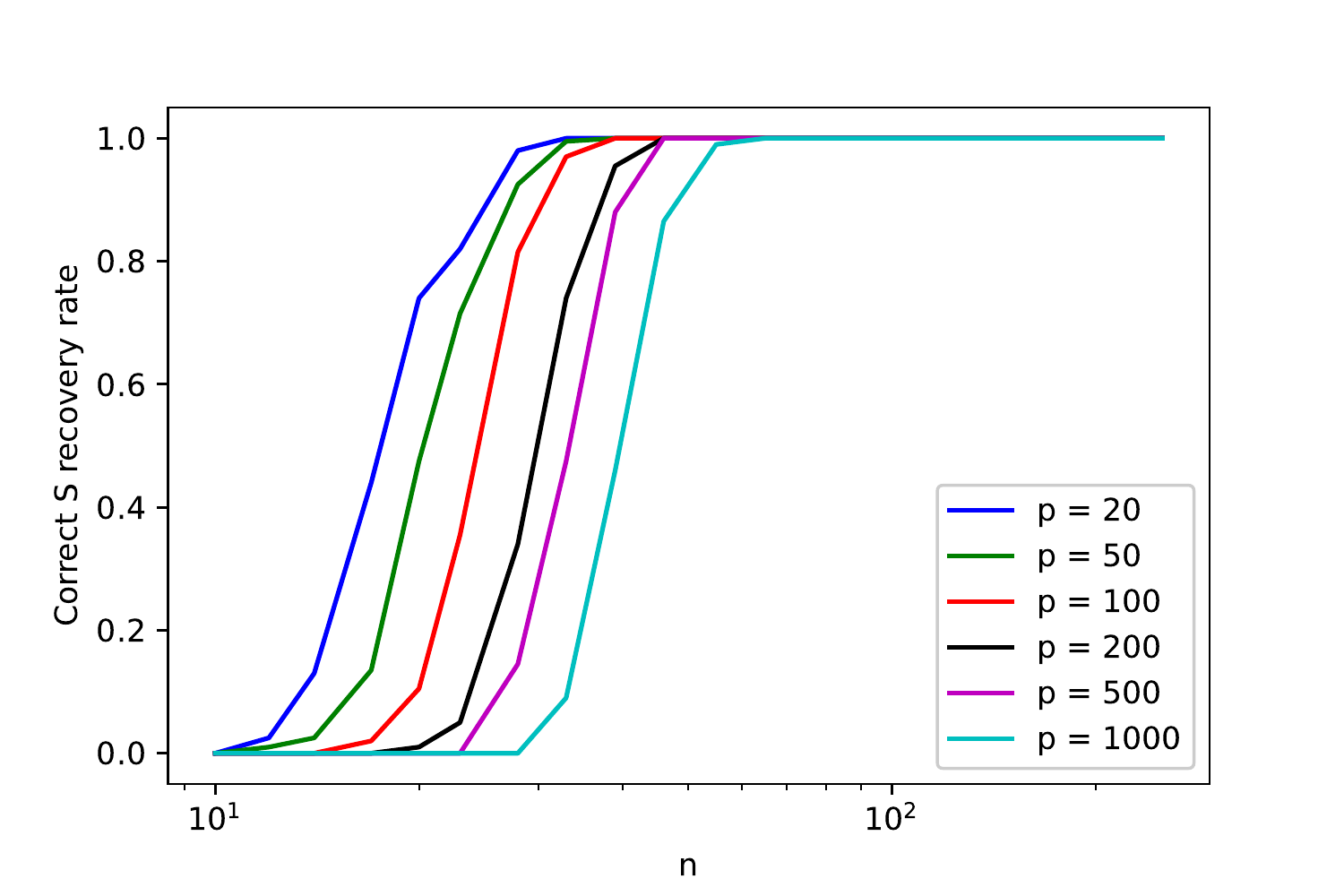}
    \caption{Empirical probability of our joint sparse algorithm (upper) and independent sparsity approach (lower) correctly recovering the sparse set $S$ as a function of $n$ and $p$, for $q=40$. }
    \label{fig:multiary40}
\end{figure}

\subsection{Additional real data experiments}
\paragraph{Cattaneo2.} The main text in Table \ref{tab:cattaneo} showed results for regularization parameter chosen via cross validation. We now consider robustness of the effect estimation to misspecification of $\lambda$. Table \ref{tab:cattaneo_sparse} shows results for $\lambda$ chosen too high (yielding very sparse $S$ with average $|S|$ of 3) and too low (yielding nonsparse $S$ with average $|S|$ of 15). Both estimates perform somewhat worse than the results in the main text, but still better than the nonsparse estimate (again shown in the main text), indicating that our approach still tends to select useful covariates.

\begin{table}
\begin{tabular}{|p{33mm}|p{20mm}|p{20mm}|}
\hline
 &  {Sparse DR Estimate (Ours, too sparse)} & {Sparse DR Estimate (Ours, less sparse)}  \\\hline
Effect of 1 vs. 0  & -217.1g(21.2) & -152.2g(24.7) \\\hline
Effect of 2 vs. 0  & -279.4g(17.4) & -195.6g(34.8)\\\hline
Effect of 3 vs. 0  & -302.9g(17.5) & -197.0g(34.0) \\\hline 
Binary effect($>0$ vs 0) &  -269.1g(14.7) & -194.7g(31.8) \\\hline \end{tabular}
\caption{Estimated average treatment effects on Cattaneo2 dataset, showing for larger regularization yielding sparser $S$ (on average cardinality of 3) and smaller regularization yielding less sparse $S$ (average cardinality of 15). Compare to Table \ref{tab:cattaneo} in the main text. Actions -- 0: no smoking, 1: 1-5 cigarettes daily, 2: 6-10 daily, and 3: 11 or more. For binary action effect, the empirical estimated interval is known to be  (-250g, -200g). Standard deviations over 20 random data splits are given in parentheses. }
\label{tab:cattaneo_sparse}
\end{table}

\paragraph{IDHP.} For the IDHP data, we know that $S$ is sparse (since the datset is semisynthetic), but we aren't told about the sparsity of the set $X_1 \cup X_2$. To answer this question, we used the doubly robust estimator with covariates selected as those 12 (out of 25) with the largest magnitude coefficients when regressing treatment $T$ versus $X$. The resulting treatment effect estimate was 5.61, with variance 0.623. This is actually not only worse than our approach, but worse than the nonsparse estimate as well (see main text).

\section{Proof of Lemma \ref{thm:admiss}}\label{app:admiss}
\begin{proof}
Chapter 11 of \cite{pearl2009causality} 
gives two sufficient conditions for strong ignorability and $c$-equivalence. If $A$ and $A'$ are two sets of covariates, then if either of
\begin{align*}
\mathrm{(a)}&\qquad T \perp A' | A,\quad \mathrm{and}\quad Y \perp A|T,A',\\
\mathrm{(b)}&\qquad T \perp A | A',\quad \mathrm{and}\quad Y \perp A'|T,A
\end{align*}
are satisfied, then $A'$ is $c$-equivalent to $A$ and we can replace $A$ with $A'$ in the treatment effect estimation.

Let us use the graph in Figure \ref{fig:graph} to check the $c$-equivalence of $S$ to $X$, using condition (a).
\begin{enumerate}
 \vspace{-1mm}
    \item $T\perp S | X$ immediately since $S$ is a subset of $X$.
     \vspace{-1mm}
    \item $Y \perp X | T, S$ holds since the graph indicates that $T,S$ form a Markov blanket for $Y$.
     \vspace{-1mm}
\end{enumerate}
We also verify it for $X_1 \cup X_2$, using condition (b):
\begin{enumerate}
 \vspace{-1mm}
    \item $T\perp X | (X_1 \cup X_2)$ holds since the graph indicates that $T,X_1\cup X_2$ form a Markov blanket for $T$.
     \vspace{-1mm}
    \item $Y \perp (X_1 \cup X_2) | T, X$ immediately since $X_1 \cup X_2$ is a subset of $X$.
     \vspace{-1mm}
\end{enumerate}
\end{proof}

\section{Proof of Lemma \ref{lem:stepOne}}\label{app:stepOne}
We first state the following lemma, which allows us to use the first of the two joint RSC conditions.
\begin{lemma}\label{lem:normLessOne}
Suppose $\hat{\theta}$ is a zero subgradient point of the objective \eqref{eq:objfun} supported on $S$, i.e.
\begin{equation}\label{eq:zerosubFirst}
\nabla \mathcal{L}_n(\hat{\theta}_S) + \nabla \rho_\lambda(\hat{\theta}_S) = 0.
\end{equation}
Let $\tilde{\nu} := \hat{\theta} - \theta^*$. Then $\|\tilde{\nu}\|_F \leq 1$. 
\end{lemma}

Lemma \ref{lem:normLessOne} implies that $\|\hat{\theta}_S - \theta^*_S\|_{F} \leq 1$. Hence the first joint RSC condition \eqref{eq:RSC} applies, so we have
\begin{equation}
\langle \nabla \mathcal{L}(\hat{\theta}_S) - \nabla \mathcal{L}(\theta^*_S), \tilde{\nu} \rangle \geq\alpha_1 \|\tilde{\nu}\|_F^2 - \tau_1{ \frac{\log k}{n}} \|\tilde{\nu}\|_{1,2}^2. 
\label{eq:21b}
\end{equation}
We also have, by the convexity of $\rho_\lambda(\theta) + \mu/2 \|\theta\|_F^2$ implied by the $\mu$-amenability of $\rho_\lambda$, that
\begin{equation}
\langle \nabla \rho_\lambda(\hat{\theta}_S), \theta^\ast_S - \hat{\theta_S}\rangle \leq \rho_\lambda(\theta^\ast_S) - \rho_\lambda(\hat{\theta}_S) + \frac{\mu}{2} \|\tilde{\nu}\|_F^2.
\label{eq:22}
\end{equation}
We know that since $\hat{\theta}$ is a stationary point, $\langle \mathcal{L}_n(\hat{\theta}_S) + \nabla \rho_\lambda(\hat{\theta}_S) , \theta_S - \hat{\theta}_S\rangle \geq 0$ for all feasible $\theta$. Using this fact with \eqref{eq:21b} and \eqref{eq:22} yields
\begin{align}
&(\alpha_1 - \mu/2) \|\tilde{\nu}\|_F^2  \nonumber\\&\leq - \langle \nabla \mathcal{L}_n(\theta^\ast_S), \tilde{\nu}\rangle + \rho_\lambda(\theta^\ast_S) - \rho_\lambda(\hat{\theta}_S)+ \tau_1 \frac{\log k}{n} \|\tilde{\nu}\|_{1,2}^2\nonumber\\
&\leq \rho_\lambda(\theta^\ast_S) - \rho_\lambda(\hat{\theta}_S) + \left(\|\nabla \mathcal{L}_n(\theta^\ast_S)\|_{\infty,2} +R\tau_1 \frac{\log k}{n}\right)\|\tilde{\nu}\|_{1,2},\label{eq:23new}
\end{align}
where we have again applied \eqref{eq:Holder}.

Now by \eqref{eq:nabla} and the fact that $\tau_1 =q$ by Lemma \ref{lem:RSC}, we have 
\begin{align}
\|\nabla& \mathcal{L}_n(\theta^\ast_S)\|_{\infty,2} +R\tau_1 \frac{\log k}{n} \nonumber\\\nonumber
&\leq c' \sqrt{\frac{q \log p}{n}} + \sqrt{\frac{R^2 q \log k }{n}} \sqrt{\frac{q \log p}{n}}\\
&\leq \frac{\lambda}{2} + \frac{\lambda }{2}
= \lambda,\label{eq:lmdabd}
\end{align}
where we have used the assumptions that $\lambda \geq c_\ell \sqrt{\frac{q \log p}{n}}$ and $n \geq C R^2 q \log p$ where here we require $c_\ell \geq 2c'$ and $C \geq \frac{1}{4c^2_\ell}$.

Also recall that the definition of $(\mu,\gamma)$ amenability states that the function $\rho_\lambda + \frac{\mu t^2}{2}$ is convex over the real line, $\lim_{t\rightarrow 0^+}\rho'_\lambda(t) = \lambda$, and $\rho_\lambda$ is symmetric about 0. Combining these facts implies that for scalar $t$, $\lambda |t| \leq \rho_\lambda(t) + \frac{\mu t^2}{2}$. This in turn implies by substitution that $\lambda \|\theta\|_{1,2} \leq \rho_\lambda(\theta) + \frac{\mu \|\theta\|_{F}^2}{2}$.

We use this fact, the subadditivity of $\rho_\lambda$ (implied by the condition that $\frac{\rho_\lambda(t)}{t}$ is nonincreasing on $\mathbb{R}^+$), and the inequality \eqref{eq:lmdabd} to simplify \eqref{eq:23new} as
\begin{align*}
&(\alpha_1 - \mu/2) \|\tilde{\nu}\|_F^2 \leq \rho_\lambda(\theta^\ast_S) - \rho_\lambda(\hat{\theta}_S) + {\lambda} \|\tilde{\nu}\|_{1,2}\\
&\leq \rho_\lambda(\theta^\ast_S) - \rho_\lambda(\hat{\theta}_S) + {\lambda} \left(\rho_\lambda(\tilde{\nu})/\lambda + \frac{\mu}{2\lambda} \|\tilde{\nu}\|_F^2\right)\\
&\leq \rho_\lambda(\theta^\ast_S) - \rho_\lambda(\hat{\theta}_S) + {\lambda} \left((\rho_\lambda(\hat{\theta}_S) + \rho_\lambda(\theta_S^\ast))/\lambda + \frac{\mu}{2\lambda} \|\tilde{\nu}\|_F^2\right)\\
&= 2 \rho_\lambda(\theta^\ast_S) + \frac{\mu}{2} \|\tilde{\nu}\|_F^2,
\end{align*}
hence $0 \leq (\alpha_1 - \mu)\|\tilde{\nu}\|_F^2 \leq 2 \rho_\lambda(\theta^\ast_S) \leq 2 \lambda \|\theta_S^\ast\|_{1,2} \leq R\lambda$, implying that 
\[
\|\tilde{\nu}\|_F \leq \sqrt{\frac{R\lambda}{\alpha_1 - \mu}}
\]
and thus via a norm inequality
\[
\|\tilde{\nu}\|_{1,2} \leq \sqrt{\frac{Rk\lambda}{\alpha_1 - \mu}}.
\]


By the triangle inequality we then have
\[
\|\hat{\theta}_S\|_{1,2} \leq \|\theta^\ast\|_{1,2} + \|\hat{\theta}_S - \theta^\ast_S\|_{1,2} \leq \frac{R}{2} + \sqrt{\frac{Rk\lambda}{\alpha_1 - \mu}} < R.
\]
where the last inequality follows by the fact that $R > \frac{4 k\lambda}{\alpha_1 - \mu}$ under our assumptions.\qed

\section{Proof of Lemma \ref{lem:RSC}}\label{app:lemRSC}

By the proof of Corollary 1 in \cite{loh2015regularized} and using the fact that our loss function \eqref{eq:lossfunc} decouples across columns, we have that with probability at least $1 - q c_1 \exp(-cn)$ and $n \geq O(k \log p)$, 
\begin{equation*}
\langle \nabla \mathcal{L}(\theta + \Delta) - \nabla \mathcal{L}(\theta), \Delta \rangle \geq  \frac{1}{2}\min_j(\lambda_{\min}(\Sigma_j)) \|\Delta\|_F^2 -  \frac{\log p}{n} \sum_j\|\Delta_{:j}\|_{1}^2. 
\end{equation*}

We require the following lemma.
\begin{lemma}\label{lem:normIneq}
For $A \in \mathbb{R}^{p\times q}$, $\|A\|_{1,2} \geq \frac{1}{\sqrt{q}}\|A^T\|_{2,1}$. 
\end{lemma}
\begin{proof}
We have $\|A\|_{1,2} = \sum_i \|A_{i:}\|_2$ and $\|A^T\|_{2,1} = \sqrt{\sum_j \|A_{:j}\|_1^2}$. Note that
\[
\|A\|_{2,1} \leq \sqrt{q}\|A\|_F \leq \sqrt{q} \|A\|_{1,2}.
\]
\end{proof}

Applying Lemma \ref{lem:normIneq}, we have
\begin{equation*}
\langle \nabla \mathcal{L}(\theta + \Delta) - \nabla \mathcal{L}(\theta), \Delta \rangle \geq  \frac{1}{2}\min_j(\lambda_{\min}(\Sigma_j)) \|\Delta\|_F^2 -  \frac{q \log p}{n} \|\Delta\|_{1,2}^2,
\end{equation*}
as desired for the $\|\Delta\|_{F} \leq 1$ case.

If $\|\Delta\|_{F} \geq 1$, then by the constraint $\|\Delta\|_{1,2} \leq R$ and assumption $n \geq 4R^2 q\log p$ we have 
\[
\frac{1}{2}\min_j(\lambda_{\min}(\Sigma_j)) \|\Delta\|_F^2 -  \frac{{q} \log p}{n} \|\Delta\|_{1,2}^2 \geq \frac{1}{2}\min_j(\lambda_{\min}(\Sigma_j)) \|\Delta\|_F - \sqrt{ \frac{q\log p}{n}} \|\Delta\|_{1,2}.
\]

Moving onto the second part of the lemma, we have (since $\mathcal{L}_n$ is the least squares loss) that
\[
\nabla^2 \mathcal{L}_n(\theta) = \mathrm{diag}\left(\left\{\frac{X_j^T X_j}{n}\right\}_{j=1}^q\right),
\]
where $\mathrm{diag}$ indicates the block diagonal matrix formed with the given blocks. Now since the $X_j$ are subgaussian with covariance $\Sigma^{(j)}_x$, we have that (Proposition 2.1 of \cite{vershynin2012close}) 
\[
|||((1/n)[X_j^T X_j]_{SS}) - ([\Sigma^{(j)}_x]_{SS})|||_2 \leq  |||\Sigma^{(j)}_x|||_2 \sqrt{\frac{k \log p}{n}}
\]
with probability at least $1 - c_1 \exp(- c_2 \log p)$. Since we have assumed that $\lambda_{\min}([\Sigma^{(j)}_x]_{SS}) > 2\mu$, we therefore have
\[
\lambda_{\min}(([X_j^T X_j]_{SS}/n) \geq 2 \mu - \mu > \mu
\]
for $n > \frac{k \log p |||\Sigma^{(j)}_x|||_2^2}{\mu^2}$.

With the union bound we thus have that the function $\mathcal{L}_n(\theta_S) - \frac{\mu}{2} \|\theta_S\|_F^2$ is strictly convex with probability at least $1 - c_1 q\exp(- c_2 \log p)$. By the definition of $(\mu,\gamma)$ amenability, we known that $\rho_\lambda - \frac{\mu}{2}t^2$ is convex. Since the addition of a strictly convex function and a convex function is strictly convex, the lemma results.
\qed

\section{Proof of Lemma \ref{lem:normLessOne}}\label{app:lemnormLessOne}
Suppose $\|\tilde{\nu}\|_{F} > 1$. Then by joint RSC \eqref{eq:RSC} we have
\[
\langle \nabla \mathcal{L}(\hat{\theta}) - \nabla \mathcal{L}(\theta^*), \tilde{\nu} \rangle \geq\alpha_2 \|\tilde{\nu}\|_F - \tau_2\sqrt{ \frac{\log p}{n}} \|\tilde{\nu}\|_{1,2}. 
\]
Since $\hat{\theta}$ is a stationary point, $\nabla \mathcal{L}(\hat{\theta}) + \nabla \rho_\lambda(\hat{\theta}) = 0$ and we thus have
\begin{equation}
\langle -\nabla \rho_\lambda(\hat{\theta}) - \nabla \mathcal{L}(\theta^*), \tilde{\nu} \rangle \geq\alpha_2 \|\tilde{\nu}\|_F - \tau_2\sqrt{ \frac{\log p}{n}} \|\tilde{\nu}\|_{1,2}. 
\label{eq:20}
\end{equation}
Recall that for equal sized matrices $A,B$
\begin{align}
\nonumber \langle A,B\rangle  
&= \sum_{i} \langle A_{i:} , B_{i:} \rangle\\
\nonumber&\leq \sum_i \|A_{i:}\|_2 \|B_{i:}\|_2\\
\nonumber&\leq \left(\max_i \|A_{i:}\|_2\right) \left(\sum_i \|B_{i:}\|_2\right)\\
&= \|A\|_{\infty,2} \|B\|_{1,2},
\label{eq:Holder}
\end{align}
where for both inequalities we have applied Holder's inequality. We can then write
\begin{equation}
\langle -\nabla \rho_\lambda(\hat{\theta}) - \nabla \mathcal{L}(\theta^*), \tilde{\nu} \rangle \leq \left( \|\nabla \rho_\lambda(\hat{\theta})\|_{\infty,2} + \|\nabla \mathcal{L}(\theta^*)\|_{\infty,2} \right) \|\tilde{\nu}\|_{1,2} \leq \left(\lambda + \lambda/2\right)\|\tilde{\nu}\|_{1,2},
\label{eq:21}
\end{equation}
where the last inequality follows from the definition of $\rho_\lambda$ and applying \eqref{eq:nabla} in the main text that yields $\|\nabla \mathcal{L}(\theta^*)\|_{\infty,2} \leq \lambda/2$ when $c_\ell \geq 2c'$. 

Combining \eqref{eq:21} with \eqref{eq:20} yields 
\begin{align*}
\alpha_2 \|\tilde{\nu}\|_F - \tau_2\sqrt{ \frac{\log p}{n}} \|\tilde{\nu}\|_{1,2} &\leq 1.5\lambda \|\tilde{\nu}\|_{1,2},\\
\|\tilde{\nu}\|_F &\leq \frac{\|\tilde{\nu}\|_{1,2}}{\alpha_2} \left(1.5\lambda + \tau_2\sqrt{ \frac{\log p}{n}}\right) \\&\leq \frac{2R}{\alpha_2} \left(1.5\lambda + \tau_2\sqrt{ \frac{\log p}{n}}\right)\\
&\leq \frac{2R}{\alpha_2} \left(1.5\frac{c_u}{R} + \tau_2\sqrt{ \frac{\log p}{n}}\right).
\end{align*}
Note that the right hand side is $\leq 1$ when $c_u$ is chosen satisfying $c_u \geq \frac{\alpha_2}{6}$ and $n \geq \frac{16}{\alpha_2^2} R^2\tau_2^2 \log p$ (since $\tau_2 = \sqrt{q}$, corresponds to having $C \geq \frac{16}{\alpha_2^2}$ in the statement of Theorem \ref{thm:consistency}), yielding a contradiction with our earlier assumption. \qed

\section{Proof of Lemma \ref{lem:5}}\label{app:lem5}
We have for all $i \in S$
\[
\|\hat{\theta}_{i:}\|_2 \geq \|\theta^\ast_{i:}\|_2 - |\langle\hat{\theta}_{i:} - \theta^\ast_{i:}, \theta^\ast_{i:}/ \|\theta^\ast_{i:}\|_2\rangle|. 
\]
Now by an easy extension of the argument in Appendix D.1.1 of \cite{loh2017support}, we have that
\[
\max_i |\langle\hat{\theta}_{i:} - \theta^\ast_{i:}, \theta^\ast_{i:}/ \|\theta^\ast_{i:}\|_2\rangle| \leq c_3 \sqrt{\frac{\log p}{n}}
\]
with probability at least $1 - c_1 \exp (-c_2 \min(k,\log p))$. We then have
\[
\|\hat{\theta}_{i:}\|_2 \geq \lambda \gamma + c_3 \sqrt{\frac{\log p}{n}} - c_3 \sqrt{\frac{\log p}{n}} = \lambda \gamma.
\]
Recall that by Definition \ref{def:amen} of $(\mu,\gamma)$ amenability, we have that $\rho'_\lambda(t) = 0$ for all $t \geq \gamma \lambda$. \qed

\section{Proof of Lemma \ref{lem:3}}\label{app:lem3}

Define $\tilde{\nu} = \tilde \theta - \hat \theta$, where recall $\hat \theta$ is the oracle estimate \eqref{eq:objfun}. We will show that $\|\tilde{\nu}\|_F\leq 1$. By contradiction, suppose that $\|\tilde{\nu}\|_F> 1$. Then by the RSC condition \eqref{eq:RSC}
\[
\langle \nabla \mathcal{L}_n (\tilde \theta) - \nabla \mathcal{L}_n(\hat \theta)\rangle \geq \alpha_2 \|\tilde\nu\|_F - \tau_2 \sqrt{\frac{\log p}{n}}\|\tilde \nu\|_{1,2}.
\]
Since both $\hat{\theta}$ and $\tilde \theta$ are stationary points and $\hat{\theta}$ is an interior local minimum (by Step 2), we have
\begin{align*}
    \langle \nabla \mathcal{L}_n(\tilde \theta) + \nabla \rho_\lambda(\tilde{\theta}), \hat{\theta} - \tilde{\theta}\rangle &\geq 0\\
    \nabla \mathcal{L}_n(\hat \theta) + \nabla \rho_\lambda(\hat{\theta}) = 0.
\end{align*}
Combining inequalities yields
\begin{align*}
    \alpha_2 \|\tilde\nu\|_F - \tau_2 \sqrt{\frac{\log p}{n}}\|\tilde \nu\|_{1,2} &\leq \langle - \nabla \mathcal{L}_n(\hat \theta) + \nabla \rho_\lambda(\tilde \theta), \tilde \nu\rangle\\
    &=\langle \nabla \rho_\lambda(\hat \theta) + \nabla \rho_\lambda(\tilde \theta), \tilde \nu\rangle\\
    &\leq (\|\nabla \rho_\lambda(\hat \theta)\|_{\infty,2} + \|\nabla \rho_\lambda(\tilde \theta)\|_{\infty,2})\|\tilde\nu\|_{1,2},
\end{align*}
where we have applied the norm inequality \eqref{eq:Holder}. Recall that by $(\mu,\gamma)$-amenability (see Lemma 8 of \cite{loh2017support}) $\|\nabla \rho_\lambda( \theta)\|_{\infty,2} \leq \lambda$ for any $\theta$. Hence we can rearrange and obtain
\[
\|\tilde\nu\|_F \leq \frac{\|\tilde \nu\|_{1,2}}{\alpha_2} \left(2\lambda + \tau_2 \sqrt{\frac{\log p }{n}}\right) \leq \frac{2R}{\alpha_2 }\left(2\lambda + \tau_2 \sqrt{\frac{\log p }{n}}\right)
\]
due to the norm constraint on the objective \eqref{eq:nonconvex}. Since we have assumed $\lambda \leq \frac{\alpha_2}{8R}$ and $n \geq \frac{16}{\alpha^2_2} R^2 \tau_2^2\log p$, $\|\tilde{\nu}\|_F\leq 1$ as desired.

We can then apply the appropriate RSC condition from \eqref{eq:RSC} yielding
\[
\langle \nabla \mathcal{L}_n (\tilde \beta) - \nabla \mathcal{L}_n (\hat \beta),\tilde \nu \rangle \geq \alpha_1 \|\tilde \nu \|_2^2  - \tau_1 \frac{\log p}{n} \|\tilde \nu\|_{1,2}^2,
\]
and (recalling the definition of $\bar{\mathcal{L}}_n$ from \eqref{eq:shifted})
\begin{equation}\label{eq:52}
\langle \nabla \bar{\mathcal{L}}_n (\tilde \beta) - \nabla \bar{\mathcal{L}}_n (\hat \beta),\tilde \nu \rangle \geq (\alpha_1 - \mu) \|\tilde \nu \|_2^2  - \tau_1 \frac{\log p}{n} \|\tilde \nu\|_{1,2}^2.
\end{equation}
By the first order optimality conditions we have
\begin{align*}
\langle \nabla \bar{\mathcal{L}}_n (\tilde \theta), \hat \theta - \tilde \theta \rangle + \lambda \langle \tilde z , \hat\theta - \tilde\theta\rangle &= 0,\\
\langle \nabla \bar{\mathcal{L}}_n (\hat \theta),  \tilde\theta - \hat \theta \rangle + \lambda \langle \hat z , \tilde\theta - \hat\theta\rangle &= 0,
\end{align*}
where $\tilde z \in \partial \|\tilde\theta\|_{1,2}$.
Combining these and using the definition of subgradient yields
\begin{align}
&\langle \nabla \bar{\mathcal{L}}_n (\hat \theta) - \nabla \bar{\mathcal{L}}_n (\tilde \theta),  \tilde\theta - \hat \theta \rangle + \lambda \langle \hat z , \tilde\theta\rangle - \lambda \|\hat\theta\|_{1,2} + \lambda \langle \tilde z , \hat\theta\rangle - \lambda \|\tilde\theta\|_{1,2} \geq 0,\nonumber\\\nonumber
& \lambda \|\tilde\theta\|_{1,2} -\lambda \langle \hat z , \tilde\theta\rangle \leq \langle \nabla \bar{\mathcal{L}}_n (\hat \theta) - \nabla \bar{\mathcal{L}}_n (\tilde \theta),  \tilde\theta - \hat \theta \rangle + \lambda \| \tilde z\|_{\infty,2} \|\hat\theta\|_{1,2} - \lambda \|\hat\theta\|_{1,2},\\\nonumber
&\lambda \|\tilde\theta\|_{1,2} -\lambda \langle \hat z , \tilde\theta\rangle \leq \langle \nabla \bar{\mathcal{L}}_n (\hat \theta) - \nabla \bar{\mathcal{L}}_n (\tilde \theta),  \tilde\theta - \hat \theta \rangle, \\\label{eq:54}
&\lambda \|\tilde\theta\|_{1,2} -\lambda \langle \hat z , \tilde\theta\rangle \leq \tau_1 \frac{\log p}{n} \|\tilde \nu\|_{1,2}^2 - (\alpha_1 - \mu) \|\tilde \nu \|_F^2,
\end{align}
where we have used the fact that $\| \tilde z\|_{\infty,2}\leq 1$ since $\tilde \theta$ is feasible and applied the bound \eqref{eq:52}.

We also have the following result. 
\begin{lemma}\label{lem:4}
If $\lambda \geq \frac{4R\tau_1 q \log p}{\delta n}$ and $\|\hat{z}_{S^c}\|_{\infty,2} \leq 1 - \delta$, then 
\[
\|\tilde \nu\|_{1,2} \leq \left(\frac{4}{\delta} + 2\right) \sqrt{k} \|\tilde \nu\|_{F}.
\]
\end{lemma}
\begin{proof}
Applying \eqref{eq:52} to \eqref{eq:54} yields
\begin{equation}\label{eq:56a}
\lambda \langle \hat z , \tilde\theta\rangle  + \lambda \langle \tilde z , \hat\theta\rangle - \lambda \|\tilde\theta\|_{1,2} \geq \langle \nabla \bar{\mathcal{L}}_n (\tilde \theta) - \nabla \bar{\mathcal{L}}_n (\hat \theta),  \tilde{\nu} \rangle \geq (\alpha_1 - \mu) \|\tilde \nu \|_2^2  - \tau_1 \frac{\log p}{n} \|\tilde \nu\|_{1,2}^2.
\end{equation}
Recalling that $\hat{\beta}$ is supported on $S$ and $\|\tilde z\|_{\infty,2} \leq 1$, we can also write
\begin{equation}\label{eq:56b}
\lambda \langle \tilde z , \hat\theta\rangle - \lambda \|\tilde\theta\|_{1,2} \leq \lambda (\|\hat\theta\|_{1,2} -  \|\tilde\theta_S\|_{1,2}- \|\tilde\theta_{S^c}\|_{1,2}) \leq \lambda (\|\tilde{ \nu}_S\|_{1,2} - \|\tilde{\nu}_{S^c}\|_{1,2}).
\end{equation}
Additionally we can use the norm inequality \eqref{eq:Holder} to bound
\begin{align}
\lambda \langle \hat z , \tilde\nu\rangle &= \lambda \langle \hat{z}_S , \tilde{\nu}_S\rangle + \lambda \langle \hat{z}_{S^c} , \tilde{\nu}_{S^c}\rangle\nonumber\\
&\leq \lambda(\|\hat{z}_S\|_{\infty,2} \|\tilde{\nu}_{S}\|_{1,2}+\|\hat{z}_{S^c}\|_{\infty,2} \|\tilde{\nu}_{S^c}\|_{1,2}) \nonumber\\
&\leq \lambda(\|\tilde{\nu}_{S}\|_{1,2} + (1-\delta)\|\tilde{\nu}_{S^c}\|_{1,2} )
\label{eq:56c}
\end{align}
where we have used the assumption $\|\hat{z}_{S^c}\|_{\infty,2} \leq 1 - \delta$ from the lemma statement.

Combining \eqref{eq:56a}, \eqref{eq:56b}, and \eqref{eq:56c} yields
\[
- \tau_1 \frac{\log p}{n} \|\tilde \nu\|_{1,2}^2 \leq (\alpha_1 - \mu) \|\tilde \nu \|_2^2  - \tau_1 \frac{\log p}{n} \|\tilde \nu\|_{1,2}^2 \leq \lambda(2\|\tilde{\nu}_{S}\|_{1,2} -\delta\|\tilde{\nu}_{S^c}\|_{1,2} ).
\]
Our assumption on $\lambda$ implies that $\tau_1 \frac{\log p}{n} \|\tilde \nu\|_{1,2} \leq 2R\tau_1 \frac{\log p}{n} \leq \frac{\delta}{2} \lambda$, yielding
\[
-\frac{\delta}{2} \lambda \|\tilde \nu\|_{1,2} \leq \lambda(2\|\tilde{\nu}_{S}\|_{1,2} -\delta\|\tilde{\nu}_{S^c}\|_{1,2} )
\]
or equivalently
\[
\frac{\delta}{2} \|\tilde{\nu}_{S^c}\|_{1,2} \leq \left(2 + \frac{\delta}{2}\right)\|\tilde{\nu}_S\|_{1,2}.
\]
We can then write (using a norm inequality)
\[
\|\tilde \nu\|_{1,2} = \|\tilde{\nu}_{S}\|_{1,2} + \|\tilde{\nu}_{S^c}\|_{1,2} \leq \|\tilde{\nu}_{S}\|_{1,2}\left(1 + \frac{4}{\delta} + 1 \right) \leq \left(2 + \frac{4}{\delta} \right) \sqrt{k} \|\tilde{\nu}\|_{F}.
\]
\end{proof}

Recall we have assumed $\frac{c_u \sqrt{q}}{R} \geq \lambda \geq c_\ell \sqrt{\frac{q \log p}{n}}$, implying for our choices of $\delta = 1/2$ and $c_\ell, c_u$
\begin{align*}
\lambda &\geq c_\ell \sqrt{\frac{q \log p}{n}}\\
&= c_\ell \sqrt{\frac{q \log p}{n}} \frac{R}{c_u \sqrt{q}} \frac{c_u\sqrt{q}}{R}\\
&\geq \frac{R c^2_\ell}{c_u \sqrt{q}} \frac{q \log p}{n}\\
&= \frac{4R \tau_1 \sqrt{q} \log p}{\delta n}.
\end{align*}
Thus we can apply Lemma \ref{lem:4} to \eqref{eq:54}, and have
\[
\lambda \|\tilde\theta\|_{1,2} -\lambda \langle \hat z , \tilde\theta\rangle \leq \tau_1 \frac{k \log p}{n}\left(\frac{4}{\delta} + 2\right)^2 \|\tilde\nu\|_F^2 - (\alpha_1 - \mu) \|\tilde \nu \|_2^2.
\]
If $n \geq \frac{2\tau_1}{\alpha_1 - \mu}\left(\frac{4}{\delta} + 2\right)^2 k \log p $, $\lambda \|\tilde\theta\|_{1,2} -\lambda \langle \hat z , \tilde\theta\rangle\leq 0$. But we know by \eqref{eq:Holder} that $\langle \hat z , \tilde\theta\rangle \leq \|\hat z\|_{\infty,2} \|\tilde \theta\|_{1,2} \leq \|\tilde \theta\|_{1,2}$ which implies $\lambda \|\tilde\theta\|_{1,2} -\lambda \langle \hat z , \tilde\theta\rangle\geq 0$. Hence we have $\lambda \|\tilde\theta\|_{1,2} -\lambda \langle \hat z , \tilde\theta\rangle = 0$ which implies $\langle \hat z , \tilde\theta\rangle = \|\tilde\theta\|_{1,2}$. Our assumption that $\|\hat{z}_{S^c}\|_{1,2} < 1$ (strictly less than 1) implies $\tilde{\theta}_{S^c} = 0$, hence $\tilde{\theta}$ is supported on $S$.\qed

\end{document}